\documentclass[12pt]{extarticle}

\usepackage{amsmath, amsthm, amssymb, hyperref, color}
\usepackage{float}
\usepackage[shortlabels]{enumitem}
\usepackage{graphicx}
\usepackage[all]{xypic}
\usepackage{makecell}
\usepackage[final]{pdfpages}
\setboolean{@twoside}{false}
\usepackage{pdfpages}
\usepackage{caption}
\usepackage{subcaption}
\usepackage{scalefnt}
\usepackage{verbatim}
\usepackage{tikz}
\usepackage{forest}
\usepackage{bm}
\usepackage{booktabs}
\usepackage{cleveref}

\oddsidemargin \evensidemargin
\newtheorem{theorem}{Theorem}
\newtheorem{proposition}[theorem]{Proposition}
\newtheorem{lemma}[theorem]{Lemma}
\newtheorem{corollary}[theorem]{Corollary}

\usepackage[ruled,vlined]{algorithm2e}
\theoremstyle{definition}

\newtheorem{example}[theorem]{Example}
\newtheorem{remark}[theorem]{Remark}

\newcommand{\RR}{\mathbb{R}}

\newcommand{\CC}{\mathbb{C}}
\newcommand{\ZZ}{\mathbb{Z}}
\newcommand{\NN}{\mathbb{N}}
\newcommand{\PD}{\operatorname{PD}}
\DeclareMathOperator{\diag}{diag}

\newcommand{\Crit}{{\rm Crit}}

\newcommand{\rk}{{\rm rk}}
\newcommand{\Sym}{{\rm Sym}}
\newcommand{\sym}{{\rm sym}}

\newcommand{\eg}{{e.g.}}
\newcommand{\ie}{{i.e.}}

\graphicspath{ {./figures/} }

\usepackage{mathtools}

\newif\ifshowcomments

\showcommentsfalse  

\title{\textbf{Geometry and Optimization\\ of Shallow Polynomial Networks}}

\author{Yossi Arjevani, Joan Bruna, Joe Kileel, Elzbieta Polak, Matthew Trager}

\date{\today}

\begin{document}
\maketitle

\begin{abstract} We study shallow neural networks with monomial activations and output dimension one. The function space for these models can be identified with a set of symmetric tensors with bounded rank. We describe general features of these networks, focusing on the relationship between width and optimization. We then consider \emph{teacher-student} problems, which can be viewed as problems of low-rank tensor approximation with respect to non-standard inner products that are induced by the data distribution. In this setting, we introduce a \emph{teacher-metric data discriminant} which encodes the qualitative behavior of the optimization as a function of the training data distribution. 
Finally, we focus on networks with quadratic activations, presenting an in-depth analysis of the optimization landscape.  In particular, we present a variation of the Eckart-Young Theorem characterizing all critical points and their Hessian signatures for teacher-student problems with quadratic networks and Gaussian training data. 
\end{abstract}
\section{Introduction}

One of the main challenges in the theory of deep learning is to explain why
neural networks can be successfully trained in practice despite the fact that
their loss is highly non-convex. In recent years, researchers have argued that
this behavior is mainly a consequence of \emph{overparameterization}, that is,
the fact that models have a large number of parameters compared to the number of 
training samples. However, while the benefit of overparameterization has been
demonstrated in many settings, precise theoretical results often rely
on unrealistic assumptions, such as infinite 
neurons~\cite{chizat2018,NEURIPS2018_5a4be1fa, pmlr-v99-mei19a},
or only apply to very
restricted models, such as linear networks~\cite{kawaguchi2016}. Additionally, several works provide sufficient conditions for the number of
parameters needed for favorable optimization~\cite{du2018d,venturi2018spurious}, but they do not consider important information such as the training data distribution or the initialization of gradient descent.

In this paper, we seek a more precise understanding of the
conditions under which local minima pose a challenge to gradient descent. Our
general setting is based on \emph{shallow polynomial networks}, that are
functions $f_{\mathcal W}:
\mathbb R^{n} \rightarrow \RR$ of the form
\begin{equation}\label{eq:network_model} f_{\mathcal W}(x) = \sum_{i=1}^r
\alpha_{i} (w_i \cdot x)^d, \qquad \mathcal W = (\alpha,w_1,\ldots,w_{{r}}) \in
\RR^r \times (\RR^{n})^{{r} },
\end{equation} 
where $d \in \NN$. We believe that shallow polynomial networks capture many essential features of the optimization properties for more general
networks --- e.g., with ReLU activations --- while also being very mathematically 
appealing.   In particular, shallow polynomial networks are intimately related
to \emph{symmetric tensor decompositions}: the set $\mathcal F_r$ of all
shallow networks of \emph{width} at most $r$ can be identified with the set of
symmetric tensors of \emph{rank} at most $r$. Though this connection has
been noted in prior works~\cite{venturi2018spurious,kileel2019expressive}, its full implications for optimization have not yet been studied in detail. Our focus on algebraic networks aligns with recent work investigating neural network architectures using algebraic-geometric tools~\cite{kohn2024function,kohn2022geometry,henry2024geometry,Kubjas_2024,finkel2024activation}.
It is worth noting that the fact that polynomial activations are not universal approximators does not represent a limitation, as our study is centered on non-convex optimization. Moreover, polynomial networks of varying degrees are, in fact, universal approximators, {as a consequence of the Stone-Weierstrass theorem (see also~\cite{yu2021arbitrary})}

We begin by revisiting the connection between tensors and shallow networks, describing properties of tensors that offer insight into these models. A key aspect that emerges from this relationship is the existence of three distinct regimes for the network width $r$: 1) for small $r$, the set $\mathcal F_r$ of shallow networks is a
low-dimensional subset of its ambient space; 2) for intermediate $r$, the set
$\mathcal F_r$ is a full-dimensional (semi-algebraic) set and yet strictly contained in its ambient
space; 3) for large $r$, the set $\mathcal F_r$
``fills'' the space of all symmetric tensors completely. 
These settings
lead to very different behaviors when optimizing training losses, with the
filling regime being the ``easiest'' in a precise sense.
The Alexander-Hirschovitz Theorem~\cite{alexander1995} from the theory of tensors effectively classifies 
the sharp transitions between these regimes for all input dimensions $n$ and
activation degrees $d$.

After discussing general properties of polynomial networks, we focus on
\emph{teacher-student} training problems{~\cite{loureiro2021learning,saglietti2022analytical}}. For any $r,s \ge 1$, 
we consider a loss function of the form
\begin{equation}\label{eq:distance_rank} 
L_{\mathcal V}(\mathcal W) =
\|f_{\mathcal W} - f_{\mathcal V}\|^2, \quad \mathcal W \in \RR^r \times
(\RR^{n})^r, \,\,  \mathcal V \in \RR^s \times (\RR^{n})^s,
\end{equation} 
where we assume that the ``teacher'' model $\mathcal V$ is fixed and that we
optimize ``student'' parameters $\mathcal W$. Here $\|\cdot\|$ denotes a functional norm {induced by an inner product}, such as an $L^2$ inner product $\langle f, g
\rangle = \mathbb E_{x\sim \mathcal D}[f(x)g(x)]$ for some distribution of
inputs $\mathcal D$. Minimizing~\eqref{eq:distance_rank} is equivalent to finding the optimal \emph{low-rank approximation} of a
given symmetric tensor, where distance is measured according to the chosen inner
product in tensor space. For $L^2$ inner products, this distance may be expressed in terms of a finite set of moments of the input data distribution, and this
applies to both discrete and continuous distributions. We explain this viewpoint in detail and describe examples of inner products arising from different classes of distributions. We also connect our setting to the broader problem of computing the distance function of points to manifolds or varieties. For that problem, it is known that qualitative changes in the optimization arise when the target point crosses a certain locus  -- known as ED discriminant~\cite{draisma2013} or the focal locus~\cite{milnor2016morse} -- associated with the algebraic variety. We adapt this notion to define a \emph{teacher-metric data discriminant}, which accounts for variations not only in the teacher but also in the metric used for computing distances. Intuitively, this discriminant encodes the qualitative behavior of the optimization landscape as a function of the training data distribution. This is a novel tool for understanding non-convex optimization problems in machine learning, which we hope to investigate further in a separate publication.

When the weights of the teacher model have special symmetries (invariance to certain permutations), critical points of \eqref{eq:distance_rank} have been identified  as examples of \emph{symmetry breaking}  \cite{arjevani2021symmetry} (see generally \cite{arjevani2024symmetry}) under the class of inner products considered in this work, enabling the construction of infinite families of critical points and the characterization of their Hessian spectrum \cite{arjevani2023symmetry}. On the other hand,  when symmetry is absent, the optimization landscape for low-rank tensor approximation is not fully understood. Existing results only consider the Frobenius norm and are typically restricted to rank-$1$ approximation~\cite{ge2017optimization, draisma2018,anandkumar2014,kileel2021landscape}.  Even in the simpler setting of low-rank matrix approximation, the behavior under weighted norms also remains challenging~\cite{gillis2011low}. For this
reason, after some general results for arbitrary activation degrees, we conduct a detailed study for the particular case of quadratic activations. 
Shallow networks with quadratic activations have been the subject of several papers~\cite{du2018d,venturi2018spurious,soltanolkotabi2018,gamarnik2020,mannelli2020}. Compared to these works, our analysis provides a more general and
detailed picture of the loss landscape. 
In particular, we provide a
characterization of \emph{all} critical points and their signatures in the case of
Gaussian (or rotationally invariant) data distributions.  Perhaps surprisingly,
we also find that the teacher-student loss for
``generic'' data distributions behaves very differently, and leads to many more
critical points than the Gaussian case. This result illustrates potential dangers of theoretical
works focusing exclusively on Gaussian data.  Throughout the paper, tools of applied algebraic geometry allow us to achieve novel analysis for shallow networks.

\paragraph{{Summary of main results.}} 

Our main contributions can be summarized as follows:
\begin{itemize}
    \item \textbf{Function space of shallow polynomial networks:}  
    We characterize the function space $\mathcal F_r$ of shallow polynomial networks through its connection with spaces of low-rank tensors. In particular, we identify three distinct geometric regimes—low-dimensional, thick, and filling—as a function of width $r$ (\Cref{subsec:functional-space}). These regimes correspond to sharply different optimization behaviors, and their boundaries are determined by the Alexander-Hirschowitz Theorem.
    \item \textbf{Geometry of loss landscapes:}  
    We analyze the structure of the loss landscape, identifying different types of unfavorable behavior—namely, ``bad minima'' and ``spurious valleys''—which may exist even in very wide networks. We also prove the existence of local minima with basins of positive Lebesgue measure  (\Cref{ex:bad_minima_exist}, \Cref{cor:positive-attraction-basin}), refining and extending prior work~\cite{venturi2018spurious}.
    \item \textbf{Teacher-student problems as low-rank tensor approximation:}  
    We formulate teacher-student training problems as low-rank tensor approximation under nonstandard inner products induced by the data distribution (\Cref{subsec:functional-norms}). These inner products are determined by the moments of the distribution, and we provide examples arising from various distribution families.
    \item \textbf{Data discriminant:}  
    We introduce the \emph{teacher-metric data discriminant}, a generalization of the focal locus (or ED discriminant), which captures qualitative changes in the optimization landscape due to variations in both the teacher model and the data distribution (\Cref{thm:disc-1}, \Cref{thm:focal_points}). We present explicit examples with symbolic computations using Macaulay2 (\Cref{subsec:examples}).
    \item \textbf{Landscape characterization for quadratic activations:}  
    We prove versions of the Eckart-Young Theorem for teacher-student problems with quadratic activations and Frobenius (\Cref{thm:eckart-young-fro}) or Gaussian (\Cref{thm:eckart-young-Gauss}) norms. We prove that for other norms, including those arising from general i.i.d. distributions, there are generally exponentially more critical points (\Cref{thm:iid-case}).
\end{itemize}

\paragraph{Notation and preliminaries.} 
The Frobenius inner product of two tensors $T$ and $S$ of the same shape is given by $\langle T, S \rangle_F = \sum_{(i_1, \ldots, i_d)} T_{i_1 \ldots i_d} S_{i_1 \ldots i_d}$. If $T = v_1 \otimes \cdots \otimes v_d$ and $S = w_1 \otimes \cdots \otimes w_d$, then $\langle T, S \rangle_F = (v_1 \cdot w_1) \cdots (v_d \cdot w_d)$, and in particular, $\langle v^{\otimes d}, w^{\otimes d} \rangle_F = (v \cdot w)^d$. A tensor $T \in \RR^n \otimes \cdots \otimes \RR^n$ (with $d$ factors) is symmetric if it is invariant under any permutation of indices, i.e., $T_{i_1, \ldots, i_d} = T_{i_{\sigma(1)}, \ldots, i_{\sigma(d)}}$ for any $\sigma \in \mathfrak{S}_d$, the symmetric group of degree $d$. We denote the space of symmetric tensors of order $d$ on $\RR^n$ by $\text{Sym}^d(\RR^n)$, a vector space of dimension $\binom{n+d-1}{d}$. This space can be identified with the space $\RR[X_1, \ldots, X_n]_d$ of homogeneous polynomials of degree $d$ in $n$ variables, where any $T \in \text{Sym}^d(\RR^n)$ corresponds uniquely to the polynomial $\langle T, X^{\otimes d} \rangle_F$ with $X = (X_1, \ldots, X_n)^T$. Finally, recall that a symmetric nonzero tensor $T \in \Sym^d(\RR^n)$ has rank $1$ if $T = \lambda v^{\otimes d}$ for some $\lambda \in \RR$ and $v \in \RR^n$. More generally, a tensor $T$ has (real symmetric) rank $r$ if it can be written as a linear combination of $r$ rank-$1$ tensors $T = \lambda_1 v_1^{\otimes d} + \ldots + \lambda_r v_r^{\otimes d}$, but not as a combination of $r-1$ rank-$1$ tensors. For $d=2$, this definition agrees with the usual notion of rank for symmetric matrices.

\section{Shallow Polynomial Networks}
\label{sec:shallow-polynomial-networks}

\textbf{Section overview.} We introduce shallow polynomial networks and their connection to symmetric tensor decompositions. We classify expressivity regimes—low-dimensional, thick, and filling—based on network width, and analyze how these transitions influence the loss landscape. In particular, we describe how critical points, bad minima, and spurious valleys depend on width, extending prior work~\cite{venturi2018spurious} and showing that unfavorable critical points can occur even in wide networks, sometimes with basins of attraction of positive measure.

\vspace{0.8em}

A \emph{shallow polynomial network} of width $r$ and with weights $\mathcal W = (\alpha, w_1,\ldots,w_r) \in \RR^r \times (\RR^n)^r$ is a function $f_{\mathcal W}: \RR^n \rightarrow \RR$ of the form
\begin{equation}\label{eq:shallow_network}
    f_{\mathcal W}(x) = \sum_{i=1}^r \alpha_i \, (w_i\cdot x)^d = \sum_{i=1}^r \alpha_i \, \langle w_i^{\otimes d},x^{\otimes d}\rangle_F = \left\langle \sum_{i=1}^r \alpha_i\, w_i^{\otimes d}, x^{\otimes d} \right\rangle_F,
\end{equation}
where $\langle T, S \rangle_F$ is the Frobenius inner product of two tensors $T, S \in (\RR^n)^{\otimes d}$. {Throughout the paper, we assume that the network has scalar output, \ie, output dimension one.} We will sometimes write the weights compactly as $\mathcal{W} = (\alpha, W)$ where $\alpha \in \mathbb{R}^r$ and $W = \begin{bmatrix} w_1 & \ldots & w_r \end{bmatrix}^T \in \mathbb{R}^{r \times n}$.

\subsection{The function space}
\label{subsec:functional-space}

We indicate the \emph{function space} of all polynomial networks of width $r$ with $\mathcal F_r$. It is clear that any network function $f_{\mathcal W}$ as in~\eqref{eq:shallow_network} is a polynomial of degree $d$ in the input $x \in \RR^n$, so we may view $\mathcal F_r$ as a subset of the space of homogeneous polynomials $\RR[X_1,\ldots,X_n]_d$ or, equivalently, of the space of symmetric tensors $\Sym^d(\RR^n)$.
Furthermore, the tensor $\sum_{i=1}^r \alpha_i w_i^{\otimes d}$ associated with $f_{\mathcal W}$ in~\eqref{eq:shallow_network} has rank at most $r$ and in fact every tensor of rank at most $r$ can be viewed as a shallow network of width $r$. The function space $\mathcal F_r$ thus coincides with the set of symmetric tensors with bounded rank:
\[
\mathcal F_r = \{f_\mathcal W \colon \mathcal W \in \RR^r \times (\RR^n)^r \} = \{\mbox{symmetric tensors of rank at most } r\} \subset \Sym^d(\RR^n).
\]
From this identification we see that the function spaces $\mathcal F_r$ are always semi-algebraic sets. It is also known that these sets {are} \emph{not closed} for $r > 2$~\cite{comon2008} and we will return to this point later in this section. It is sometimes helpful to consider the \emph{Zariski closures} $\overline{\mathcal F_r}$ of $\mathcal F_r$, as these are real algebraic varieties in $\Sym^d(\RR^n)$.

\paragraph{Thick and filling spaces.}
Basic properties of tensor ranks imply that there exist three distinct regimes for the width of a network:
\begin{enumerate}
    \item For small $r$, the set $\mathcal F_r$ is a low-dimensional subset of $\Sym^d(\RR^n)$. More precisely, $\overline{\mathcal F_r}$ is an algebraic variety strictly contained in $\Sym^d(\RR^n)$.
    \item For intermediate $r$, the set $\mathcal F_r$ is full-dimensional and yet strictly contained in its ambient space. That is, $\mathcal F_r \subsetneq \Sym^d(\RR^n)$ but $\overline{\mathcal F_r} = \Sym^d(\RR^n)$ (this does \emph{not} mean that $\mathcal F_r$ is dense for the Euclidean topology).
    \item For large $r$, $\mathcal F_r = \Sym^d(\RR^n)$ holds.
\end{enumerate}

 Similar to~\cite{kileel2019expressive}, we say that $\mathcal F_r$ is \emph{thick} if it has positive Lebesgue measure in $\Sym^d(\RR^n)$ (case $2$ and $3$ above); we say that $\mathcal F_r$ is \emph{filling} if $\mathcal F_r = \Sym^d(\RR^n)$ (case $3$ above). We also let
 \begin{equation}
 \begin{aligned}
 &r_{\rm thick}(d,n) &:=&&& \min\{r \colon \mathcal F_r \mbox{ is thick}\},\\[.25cm]
 &r_{\rm fill}(d,n) &:=&&& \min\{r \colon \mathcal F_r \mbox{ is filling}\}.\\[.15cm]
 \end{aligned}
 \end{equation}
These two widths correspond to \emph{sharp} qualitative changes in the function space. The value of $r_{\textup{thick}}(d,n)$ is known in all cases, but only an upper bound for $r_{\rm fill}(d,n)$ is known in general~\cite{bernardi2018real}.

\begin{theorem}\label{thm:thick_filling}
If $d = 2$, we have $r_{\textup{thick}}(2,n) = r_{\textup{fill}}(2,n) = n$. If $d \ge 3$, then 
    \[
r_{\rm thick}(d,n) = \left \lceil{ \frac{1}{n}\dim(\Sym^d(\RR^n))}\right \rceil = \left \lceil \frac{1}{n} \binom{n+d-1}{d} \right \rceil,
\] 
except for $(d,n) = (4,3),(4,4),(4,5),(3,5)$, when this bound needs to be increased by one.
For all $(d,n)$, it holds that $r_{\rm fill}(d,n) \le 2 r_{\rm thick}(d,n)$.  
\end{theorem}

These facts were already briefly discussed in~\cite{kileel2019expressive}, which considered a broader class of deep polynomial networks and mainly studied the transition from low-dimensional to thick function spaces (referred to as ``filling functional varieties'').  Here, we restrict attention to shallow networks but study the nonconvex optimization landscape of these models,  a topic not addressed in~\cite{kileel2019expressive}.

\begin{example} If $n=2$ (two-dimensional input), then $\mathcal F_r$ is a family of homogeneous polynomials in $\RR[X_1,X_2]_d \cong \RR^{d+1}$, or equivalently, a family of non-homogeneous polynomials in {in a single variable}. Assuming $d \ge 3$, then Theorem~\ref{thm:thick_filling} states that $\mathcal F_r$ is a full-dimensional subset of $\RR^{d+1}$ if and only if $r \ge r_{\rm thick}(d,2) = \lceil \frac{d+1}{2} \rceil$. In this case, it is also known that $r_{\rm fill}(d,2) = d$
\cite[Proposition 2.1]{comon2012typical}. 
In the simplest case of $d=3, n=2$, then we have that $\mathcal F_3$ contains all cubic polynomials in one variable, while $\mathcal F_2$ can be characterized explicitly as the set of cubic polynomials that have exactly one real root. 
\end{example}

\begin{remark}
The notions of filling and thick function spaces are also applicable to networks with \emph{non-polynomial activations} (\eg, ReLU). Indeed, while the ambient function spaces for these networks are infinite-dimensional, a finite-dimensional embedding space always exists in the practical setting of empirical
risk minimization (ERM), since functions are mapped to predictions on a
finite set of training data.
\end{remark}

\paragraph{The parameterizing map.}

Since optimization takes place in parameter space, we are interested in the parameterization of the function spaces $\mathcal F_r$:

\begin{equation}\label{eq:parameterization}
    \tau_r: \RR^r \times (\RR^n)^r \rightarrow \mathcal F_r, \quad \mathcal{W} = (\alpha, w_1,\ldots, w_r)\mapsto \sum_{i=1}^r \alpha_i w_i^{\otimes d}.
\end{equation}
We first briefly describe the {fibers} of $\tau_r$ (the preimages 
of single points). This can also be viewed as characterizing 
the {symmetries} of the parameters in the model. We note 
that the parameters of any shallow network as 
in~\eqref{eq:shallow_network} has a set of ``trivial'' 
symmetries, namely scalings among layers and permutations 
among neurons. More precisely, if $D \in \RR^{r \times r}$ is 
an invertible diagonal matrix and $P \in \ZZ^{r \times r}$ is 
a permutation matrix, then any pair of parameters $\mathcal{W} 
= (\alpha, W)$ and $\mathcal{W}' = (\alpha', W')$ that are 
related by $\alpha' = P D^{-d} \alpha$ and $W' = PDW$ {satisfies} 
$\tau_r(\mathcal W)=\tau_r(\mathcal W')$ (\ie, $f_{\mathcal W} 
= f_{\mathcal W'}$). Classical results on 
the identifiability of tensor decomposition imply that in many 
cases there are no other symmetries.

\begin{proposition} \label{prop:kruskal1}
Assume that $d \geq 3$, $r<r_{\rm thick}({d},{n})$ and that $(d,n,r)$ is none of $(6,3,9)$, $(4,4,8)$,  $(3,6,9)$. If $\mathcal{W} \in \mathbb{R}^r \times \left(\mathbb{R}^n \right)^r$ is Zariski-generic (\ie, belongs to an appropriate open dense set of the parameter space) and $\mathcal{W}'\in \mathbb{R}^r \times \left(\mathbb{R}^n \right)^r$ is arbitrary, then $\tau_r(\mathcal W) = \tau_r(\mathcal W')$ holds if and only if $\mathcal W$ and $\mathcal W'$ are related by trivial symmetries.
\end{proposition}

\begin{proof}
This follows from \cite[Theorem 1]{chiantini2017generic} and \cite[Remark 4]{chiantini2017effective}.
\end{proof}

We note that 
polynomial networks with quadratic activations ($d=2$) behave 
differently, since the parameters of the model have always 
more symmetries in addition to the trivial ones (for any $r$). 
To see this, consider $(\alpha, W) \in \RR^r \times \RR^{r 
\times n}$ such that 
$\alpha = 
(1,\ldots,1,-1,\ldots,-1)$, where $p$ elements {are}  $1$'s and $q$ 
elements are $-1$'s (any full rank matrix in $\Sym^2(\RR^n)$ 
has an element of this form in its fiber). Then we have that 
$W^T \diag(\alpha) W = \tau_r(\alpha, W) = \tau_r(\alpha, M 
W)$ for any $M \in \RR^{r \times r}$ that belongs to the 
\emph{indefinite orthogonal group} $O(p,q) = \{M  \colon M^T 
\diag(\alpha) M = \diag(\alpha)\}$. This set is a manifold 
(Lie group) of dimension $r(r-1)/2$ that contains $O(p) \times 
O(q)$.
In the Appendix, we give a description of the topologically connected components of the fiber for networks with quadratic activations (\Cref{thm:fiber_quadratic}). 

We also remark that, in general, optimization landscapes in parameter space and function space are related but distinct~\cite{trager2019pure, levin2024effect}.   Our paper derives results in both spaces.

\paragraph{Critical parameters and branch functions.}

We next give some properties of the critical locus of the 
parameterization, which is the set $\Crit({\tau_r}) \subset \RR^r 
\times \RR^{r \times n}$ where the rank of the differential of 
$\tau_r$ is not maximal:

\[
{\rm Crit}(\tau_r) = \left\{\mathcal W \colon {\rm rk}(d \tau_{{r}}(\mathcal W)) <
\dim(\mathcal F_r) \right\} \subset \RR^r \times \RR^{r \times n}.
\]
This is an algebraic subset of positive codimension in $\RR^r \times \RR^{r \times n}$. Its importance is illustrated by the following basic fact.

\begin{lemma}\label{lemma:pure-spurious} Assume $r \ge r_{\rm thick}(d,n)$ and let $L: \RR^r \times (\RR^n)^r\rightarrow \RR$ be a function of the form $L = \ell \circ \tau_r$ where $\ell: \Sym^d(\RR^n)\rightarrow \RR$ is smooth. If $\mathcal W \in \Crit(L)$, then either $\tau_k(\mathcal W) \in Crit(\ell)$ or $\mathcal W \in \Crit({\tau_r})$.
In particular, if $\ell$ is convex, all non-global local minima of $L$ must belong to the critical parameter set.
\end{lemma}
\begin{proof} If $r \ge r_{\rm thick}(d,n)$, then the 
differential $d\tau_r(\mathcal W)$ is surjective on generic points. From $d 
L(\mathcal W) = d \ell(\tau_r(\mathcal W)) \circ d 
\tau_r(\mathcal W)$, we see that in this case $d L(\mathcal W) = 0$ implies $d 
\ell(\tau_r(\mathcal W))=0$.
\end{proof}

When $r \ge r_{\rm thick}(d,n)$ the critical parameter set ${\rm Crit}(\tau_r)$ can be seen as a low-dimensional subset of ``bad'' parameters where {all} critical points except global minima are confined. Importantly, this set does not depend on the particular functions $L$ or $\ell$ but only on the architecture of the model. This provides a criterion for whether a critical point $\mathcal W^*$ reached by a local optimization method is a global minimum: if $\mathcal W^* \not 
\in {\rm Crit}(\tau_r)$, it is necessarily a global minimum; 
if $\mathcal W^* \in {\rm Crit}(\tau_r)$, then it is a not a global minimum unless $\ell$ 
admits a global minimum on ${\rm Crit}(\tau_r)$ (which would not occur for generic data).

Algebraic equations for ${\rm Crit}(\tau_r)$ can 
be obtained from minors of a Jacobian matrix or using the 
explicit form of the differential:
\begin{equation}\label{eq:differential_parameterization}
    d\tau_r(\mathcal W)(\dot \alpha, \dot W) = \sum_{i=1}^r \dot \alpha_i w_i^{\otimes d} + \sum_{i=1}^r \alpha_i \, {\rm sym}\left(w_i^{\otimes (d-1)} \otimes \dot w_i\right).
\end{equation}
where ${\rm sym}(w_i^{\otimes (d-1)} \otimes \dot w_i) := w_i^{\otimes (d-1)} \otimes \dot w_i + w_i^{\otimes (d-2)} \otimes \dot w_i \otimes w_i + \cdots + \dot w_i \otimes w_i^{\otimes (d-1)}$.

The \emph{branch locus} of the parameterization is the image of the critical locus ${\rm Br}(\tau_r) = \tau_r(\Crit(\tau_r))$. It can be interpreted as the set of ``bad functions'' that may correspond to non-global minima. Generally, the set ${\rm 
Crit}(\tau_r)$ is strictly contained in $\tau_r^{-1}({\rm 
Br}(\tau_r))$, so the fact that $\tau_r(\mathcal W) \in {\rm 
Br}(\tau_r)$ provides a necessary but not a sufficient 
condition for $\mathcal W \in {\rm Crit}(\tau_r)$. However, a necessary 
condition is useful for practical purposes.

\begin{proposition}\label{prop:branch_quadratic} Assume that $d=2$ (quadratic activations) and $r \ge n$. Then it holds that
\[
{\rm Br}(\tau_r) = \mathcal F_{n-1} = \{S \colon {\rm rk}(S) \le r_{\rm thick}(2,n) - 1 = n - 1\} \subset \Sym^2(\RR^n).
\]
In particular, if $\ell: \Sym^2(\RR^n)\rightarrow \RR$ is smooth and convex, any non-global local minimum $\mathcal W$ of $L = \ell \circ \tau_r$ satisfies $\rk(S) < n$ where $S = \tau_r(\mathcal W)$.
\end{proposition}
\begin{proof} Let $\mathcal W$ be such that $\rk(\tau_r(\mathcal W)) = n$. There exist $i_1,\ldots,i_n$ such that $\alpha_{i_k} \ne 0$ and $w_{i_k}$ are linearly independent, for otherwise one could find $v \in \RR^n$ such that $(\sum_{i=1}^r \alpha_i w_i w_i^T)v=0$. The image of $d \tau_r(\mathcal W)$ now contains
\[
{\rm Span}\{\sym(w_{i_k} \otimes u) \colon u \in \RR^n, k=1,\ldots,n\} = \sym({\rm Span}\{w_{i_k}\} \otimes \RR^n) = \Sym^2(\RR^n).
\]
This implies that $\mathcal W \not \in {\rm Crit}(\tau_r)$ and ${\rm Br}(\tau_r) \subset \mathcal F_{n - 1} = \{S \colon \rk(S) \le n-1\}$. 

Conversely, let $S \in \Sym^2(\RR^n)$ be such that $\rk(S) = s < n$. We can find $\mathcal W \in \RR^r \times \RR^{r \times n}$ such that $\tau_r(\mathcal W) = S$ and $\alpha_j=0$, $w_{j} = 0$ and for all $j = s+1,\ldots,r$. Such $\mathcal W$ belongs to ${\rm Crit}(\tau_r)$. Indeed, if $v \in \RR^n$ is such that $v \cdot w_i = 0$ for $i=1,\ldots, s$, then it holds that
\[
\langle {\rm sym}(w_i \otimes u), vv^T \rangle = 2(u \cdot v)(w_i\cdot  v) = 0, \quad \forall u \in \RR^n.
\]
This implies that $\langle M,vv^T\rangle = 0$ for all $M \in {\rm Im}(d \tau_r(\mathcal W))$, so ${\rm Im}(d \tau_r(\mathcal W)) \subsetneq \Sym^2(\RR^n)$. We conclude that $\mathcal W \in {\rm Crit}(\tau_r)$ and $\mathcal F_{n-1} \subset {\rm Br}(\mathcal F_{n-1})$
\end{proof}

The second claim of Proposition~\ref{prop:branch_quadratic} is similar to results that have previously appeared (see, \eg, \cite[Theorem 2.4]{gamarnik2020}), although these works do not mention the critical locus of the parameterization. To the best of our knowledge, the following result is new.

\begin{proposition}\label{prop:branch_n_2} Assume $n = 2$ (two-dimensional input) and $r \ge \lceil \frac{d+1}{2} \rceil$. 
Then it holds that
\[
{\rm Br}(\tau_r) = \mathcal F_{\left \lfloor \frac{d}{2} \right \rfloor} = \left \{T \colon {\rm rk}(T) \le r_{\rm thick}(d,2) - 1 = \left \lfloor \frac{d}{2} \right \rfloor \right \} \subset \Sym^d(\RR^2).
\]
In particular, if $\ell: \Sym^d(\RR^2)\rightarrow \RR$ is smooth and convex, any function $f_{\mathcal W}$ associated with a non-global local minimum $\mathcal W$ of $L = \ell \circ \tau_r$ can be expressed with at most $\left \lfloor \frac{d}{2} \right \rfloor$ neurons.
\end{proposition}
\begin{proof} 
Let $\mathcal W = (\alpha,W) \in {\rm Crit}(\tau_r)$. By Lemma~\ref{lemma:singular_point} below, there exists a non-zero $P \in \RR[X_1,X_2]_{d}$ such that $\nabla P(w_i) = 0$ for all $i$ such that $\alpha_i \ne 0$ and $w_i \ne 0$. This can only occur for $\lfloor \frac{d}{2} \rfloor = \lceil\frac{d+1}{2} \rceil - 1$ pairwise independent $w_i$. Indeed, if $P$ decomposes into linear factors as $P = \prod_{i=1}^d (a_i X_1 + b_i X_2)$ where $a_i, b_i \in \CC$ (using the Fundamental Theorem of Algebra), then $\nabla P(w) = 0$ holds if and only if $w = \lambda_i (b_i,-a_i) = \lambda_j (b_j,-a_j)$ for $i \ne j$ and $\lambda_i, \lambda_j \in \RR$ (this corresponds to a double root of $P$). If $\mathcal W$ has at most $\lfloor \frac{d}{2} \rfloor$ non-zero and pairwise independent ``neurons'', then $\rk(\tau_r(\mathcal W)) \le \lfloor \frac{d}{2} \rfloor$. This shows that ${\rm Br}(\tau_r) \subset \mathcal F_{\lfloor \frac{d}{2} \rfloor}$.
For the converse, we note that if $\rk(T) = s \le \lfloor \frac{d}{2} \rfloor$ then we can find $\mathcal W \in \RR^r \times \RR^{r \times n}$ such that $\tau_r(\mathcal W) = T$ and $\alpha_j=0$, $w_{j} = 0$ and for all $j = s+1,\ldots,r$. From the second claim of Lemma~\ref{lemma:singular_point}, such $\mathcal W$ belongs to ${\rm Crit}(\tau_r)$.
\end{proof}

\begin{lemma}\label{lemma:singular_point}
If $\mathcal W = (\alpha, W)$ is such that $\rk(d\tau_r(\mathcal W))< \dim(\Sym^d(\RR^n))$, then there exists $P \in \RR[X_1,\ldots,X_n]_{d}$ such that $\nabla P(w_i) = 0$ for all $w_i$ such that $\alpha_i \ne 0$. If $w_i = 0$ holds for all $i$ such that $\alpha_i = 0$, then the converse also holds.
\end{lemma}
\begin{proof}   The image of the differential $d\tau_{r}(\mathcal W)$ contains
\[
V = \textup{Span}\{\textup{sym}(w_i^{\otimes (d-1)} \otimes u) : \alpha_i \ne 0, u \in \mathbb{R}^n\} \subset \Sym^d(\mathbb{R}^n).
\]
If $\rk(d\tau_r(\mathcal W))<\dim(\Sym^d(\RR^n))$ then $V \ne \Sym^d(\mathbb{R}^n)$ and there exists $P \in \Sym^d(\mathbb{R}^n)\setminus\{0\}$ such that $P \in V^\perp$, where $V^\perp$ is the orthogonal space to $V$ with respect to the Frobenius inner product. It is now sufficient to note that $\langle P, \sym({X}^{\otimes (d-1)} \otimes u) \rangle_F = \nabla_{{X} } \langle P, {X}^{\otimes d} \rangle_{{F}} \cdot u $. Finally, if $w_i = 0$ whenever $\alpha_i = 0$, then $V$ is the image of $d\tau_{r}(\mathcal W)$, so the existence of $P$ means that $\rk(d\tau_r(\mathcal W))< \dim(\Sym^d(\RR^n))$.
\end{proof}

For general $n \ge 2$ and $d \ge 2$, it still holds that $\mathcal F_{r_{\rm thick}-1} \subset {\rm Br}(\tau_r)$ for $r \ge r_{\rm thick}$. This can be seen by considering parameterizations of points in $\mathcal F_{r_{\rm thick}-1}$ such that $\alpha_i=0$ and $w_i=0$ for all $i> r_{\rm thick}-1$. We do not know whether the reverse containment holds in general.

\subsection{Optimization landscapes}

We next discuss some general properties of the optimization 
landscape for polynomial networks. Specifically, we consider functions $L: \RR^r \times \RR^{r \times n} 
\rightarrow \RR$ of the form $L = \ell \circ \tau_r$ where 
$\ell: \Sym^d(\RR^n) \rightarrow \RR$ is smooth and 
convex and $\tau_r$ is the parameterization map 
from~\eqref{eq:parameterization}. 

\paragraph{Favorable landscapes.}
We spell out two possible ways in which the landscape of a non-convex function can be favorable:

\begin{itemize}
    \item We say that a continuous function $f: \RR^N \rightarrow \RR$ has \emph{no bad local minima} if every local minimum of $f$ is also a global minimum.
    \item We say that a continuous function $f: \RR^N \rightarrow \RR$ has \emph{no spurious valleys} if, for any $t \in \RR$, the image under $f$ of every connected component of the sublevel set $f^{-1}((-\infty,t]) \subset \RR^N$ is the same interval in $\RR$. If $f$ has a global minimum, this condition is equivalent to requiring that each connected component of the sublevel set contains a global minimum.
\end{itemize}

{The absence of spurious valleys ensures that from any point in the parameter space there exists a path to a global minimum along which the loss is non-increasing, a property that can be beneficial for gradient-based or other local search methods.} Most works on the loss landscape of neural networks aim to show the absence of bad local minima in the loss landscape~\cite{soltanolkotabi2018,du2018d,kawaguchi2016}. One exception is~\cite{venturi2018spurious}, which discusses conditions on the width of a network that guarantee that the loss has no spurious valleys. In particular, this work shows the following useful result, here rephrased and adapted to our setting.

\begin{theorem}[\cite{venturi2018spurious}, Theorem 8]\label{thm:venturi_spurious} If $r \ge \dim(\Sym^d(\RR^n))$, then any function $L: \RR^r \times \RR^{r \times n} \rightarrow \RR$ of the form $L = \ell \circ \tau_r$ where $\ell: \Sym^d(\RR^n)\rightarrow \RR$ is convex has no spurious valleys.
\end{theorem}

\begin{figure}[htbp]
    \centering
    \begin{subfigure}[b]{0.2\textwidth}
    \includegraphics[width=.9\textwidth]{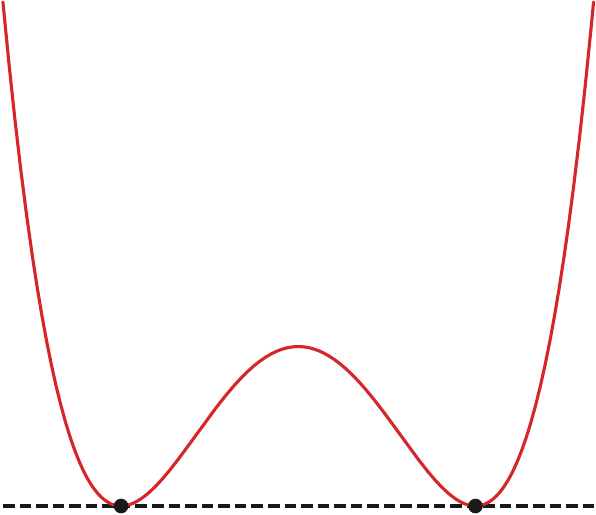} \,\,
    \caption{}
    \end{subfigure}
    \begin{subfigure}[b]{0.2\textwidth}
    \includegraphics[width=0.9\textwidth]{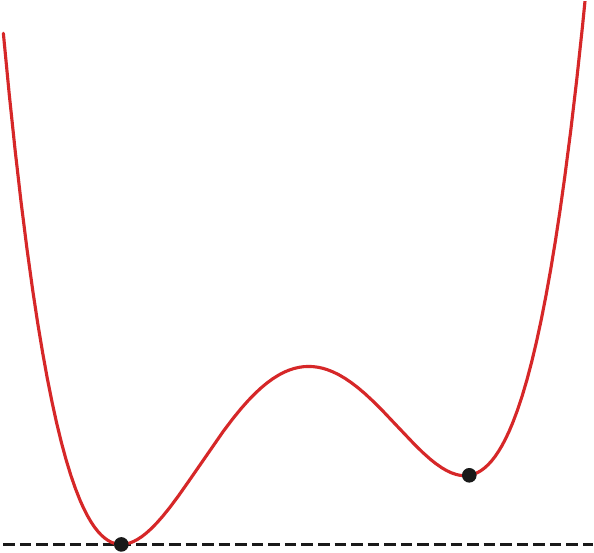} \,\,
    \caption{}
    \end{subfigure}
    \begin{subfigure}[b]{0.25\textwidth}
    \includegraphics[width=0.9\textwidth]{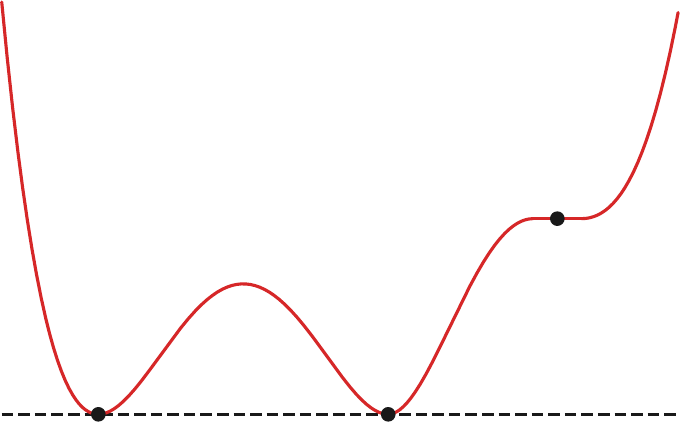} \,\,
    \caption{}
    \end{subfigure}
    \begin{subfigure}[b]{0.2\textwidth}
    \includegraphics[width=0.9\textwidth]{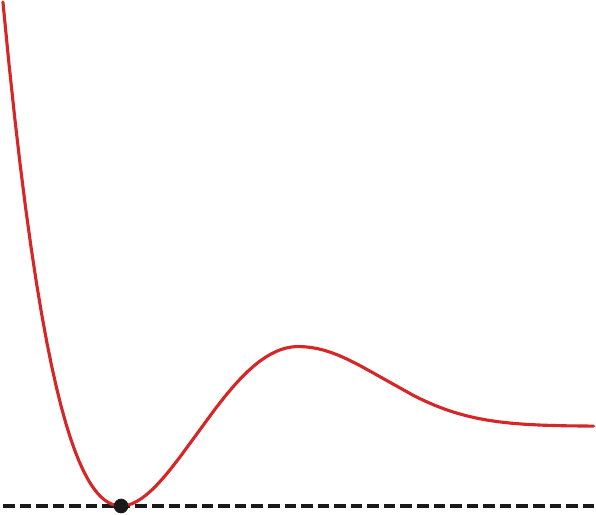}
    \caption{}
    \end{subfigure}
    \caption{Landscapes of non-convex functions: (a) no bad minima and no spurious valleys; (b) bad minima and spurious valleys; (c) bad minima and no spurious valleys (all depicted points are local minima); (d) no bad minima and spurious valleys.}
    \label{fig:landscapes}
\end{figure}

It is important to remark that there is \emph{no general implication} between the absence of bad local minima and the absence of spurious valleys (see Figure~\ref{fig:landscapes}). Indeed, a function might have a bad local minimum without a spurious valley, if the minimum is non-strict (Figure~\ref{fig:landscapes}, (c)). Conversely, and perhaps less obviously, a function can have a spurious valley without a bad local minimum, if the minimum of the valley is at infinity (Figure~\ref{fig:landscapes}, (d)).\footnote{Lemma 2 in~\cite{venturi2018spurious} makes a small mistake and claims that ``presence of spurious valleys implies existence of spurious minima.'' This is true only if all minima are finite.} We also introduce {a} third notion that is related to the absence of spurious valleys:

\begin{itemize}
    \item We say that a point $x \in \RR^N$ is \emph{escapable} for a continuous $f: \RR^N \rightarrow \RR$ if there exists a path $\gamma: [0,1] \rightarrow {\RR^N}$ such that $\gamma(0) = x$ along which $f$ does not increase (\ie, $f(\gamma(t_1)) \ge f(\gamma(t_2))$ for all $t_1 < t_2$) and such that $f(\gamma(1)) < f(x)$. Note that if $f$ is smooth, any point that does not belong to $\Crit(f)$ is escapable. 
\end{itemize}
A function that has no spurious valleys is such that all points in its domain are escapable {or global minima}. However, there exist functions that have spurious valleys but for which all points are escapable {or global minima}. An example is the function shown in Figure~\ref{fig:landscapes} (d). 

\paragraph{Landscape and width.}
We refine some results from~\cite{venturi2018spurious} that relate the optimization landscape to network width. We start by discussing a negative result: when the network architecture is not large enough, bad local minima may exist. 
According to~\cite[Corollary 14]{venturi2018spurious}, if 
$d$ is even and $L$ is the quadratic loss, then $L$ may have spurious valleys when $r \le 
\frac{1}{2}r_{\rm fill}(d,n-1)$ if $d>2$ and when $r \le n$ if 
$d=2$. We provide a more general result that considers arbitrary loss functions and applies to even and odd $d$.

\begin{proposition}\label{prop:spurious_exist} Assume that $cl(\mathcal F_r) \neq \Sym^d(\RR^n)$, where $cl$ denotes closure with respect to the Euclidean topology. This always holds for $r< r_{\rm thick}(d,n)$. Then there exists a loss function $L: \RR^r \times \RR^{r \times n}\rightarrow \RR$ of the form $L = \ell \circ \tau_r$ where $\ell: \Sym^d(\RR^n)\rightarrow \RR$ is convex such that $L$ has bad minima and spurious valleys.
\end{proposition}
\begin{proof} If $cl(\mathcal F_r) \neq \Sym^d(\RR^n)$, then $cl(\mathcal F_r)$ is not a convex subset of $\Sym^d(\RR^n)$. Indeed, for any $T \in \Sym^d(\RR^n) \setminus cl(\mathcal F_r)$ that decomposes as $T = \sum_{i=1}^s \alpha_i w_i^{\otimes d}$ (with $s>r$), we have that $T = \frac{1}{s}\sum_{i=1}^s s \alpha_i w_i^{\otimes d}$ is a convex combination of elements in $\mathcal F_1 \subset \mathcal F_r$. The fact that $cl(\mathcal F_r)$ is not convex easily implies that we can find a convex function $\ell: \Sym^d(\RR^n) \rightarrow \RR^n$ whose restriction to $\mathcal F_r$ has (arbitrarily bad) minima and spurious valleys (see~\cite[Proposition 6]{kileel2019expressive}). 
\end{proof}

Note that $\frac{1}{2}r_{\rm fill}(d,n-1) \le r_{\rm thick}(d,n-1) < r_{\rm thick}(d,n)$ so compared to~\cite{venturi2018spurious} we show that bad minima may exist for larger networks. On the positive side --- when the network is wide the landscape is favorable --- the following two results improve Theorem~\ref{thm:venturi_spurious} in two special cases.

\begin{theorem}\label{thm:escapable_n_2} Assume $n=2$ and $r \ge \lceil \frac{d+1}{2} \rceil$ (recall $r_{\rm thick}(d,2) = \lceil \frac{d+1}{2} \rceil$). Then any function of the form $L = {\ell \circ \tau_r}$ where $\ell: \Sym^d(\RR^2)$ is smooth and convex is such that all points {that are not global minima} are escapable.
\end{theorem}

\begin{theorem}\label{thm:valleys_d_2} Assume $d=2$ and $r \ge n$ (recall $r_{\rm fill}(2,n) = r_{\rm thick}(2,n)=n$). Then any function of the form $L = {\ell \circ \tau_r}$ where $\ell: \Sym^2(\RR^n) \rightarrow \RR$ is smooth and convex has no spurious valleys.
\end{theorem}

The bounds in \Cref{thm:escapable_n_2} and~\Cref{thm:valleys_d_2} improve over~\Cref{thm:venturi_spurious} by a factor of $(n+1)/2$ and $n$, respectively (\Cref{thm:escapable_n_2} only shows ``escapability'', but this is necessary in light of~\Cref{rmk:escapable} below). Moreover, these bounds are tight given \Cref{prop:spurious_exist}.

\begin{remark}\label{rmk:escapable} For $n=2$, it is known that $cl(\mathcal F_r) \ne \Sym^d(\RR^n)$ holds if and only if $r < d$ (this is because all ranks between $\lceil \frac{d+1}{2} \rceil$ and $d$ are ``typical''~\cite[Theorem 2.4]{blekherman2012}). Combining \Cref{thm:escapable_n_2} and \Cref{prop:spurious_exist}, we deduce that for $n=2$ and $\lceil \frac{d+1}{2} \rceil \le r < d$ there exist convex functions $\ell: \Sym^d(\RR^n) \rightarrow \RR$ with spurious valleys, however even for such functions all (critical) points are escapable. The landscape in these cases is analogous to the one shown in Figure~\ref{fig:landscapes}, (d). On the other hand, not all convex functions $\ell: \Sym^d(\RR^n) \rightarrow \RR$ arise as standard loss functions such as the quadratic loss considered in the next sections. It would be interesting to find concrete examples of functions with this kind of landscape (see also the related Example~\ref{ex:minimum_infinity} below).
\end{remark}

Our proof of \Cref{thm:escapable_n_2} makes use of the following simple fact.

\begin{lemma}\label{lemma:low_rank_escapable} Assume that $\mathcal W \in \RR^r \times 
\RR^{r \times n}$ is such that there exists $i\ne j$ such that 
$\alpha_i w_i$ and $\alpha_j w_j$ are scalar multiples of each other.
Then, for any $L: \RR^r \times \RR^{r \times n} \rightarrow \RR$ of the form $L = \ell \circ \tau_r$ where $\ell: \Sym^d(\RR^n)\rightarrow \RR$ is convex, $\mathcal W$ is escapable.
\end{lemma}
\begin{proof} We assume $d \ell(\tau_r(\mathcal W)) \ne 0$, as otherwise $\mathcal W$ is a global minimum. Let $v \in \RR^n$ be such that $d \ell(\tau_r(\mathcal W))(v^{\otimes d}) \ne 0$. Such $v$ exists since ${\rm Span}\{v^{\otimes d} \colon v \in \RR^n\} = \Sym^d(\RR^n)$. We observe that it is always possible to continuously change $\alpha_i, w_i, \alpha_j, w_j$ without affecting $\tau_r(\mathcal W)$ so that in the end $\alpha_i = 0$ and $w_i = v$. Indeed, we can continuously scale $\alpha_i$ to zero by modifying $\alpha_j, w_j$ so that $\alpha_i w_i^{\otimes d} + \alpha_j w_j^{\otimes d}$ (and also $\tau_r(\mathcal W)$) remains constant. Once $\alpha_i=0$, we can change $w_i$ freely without affecting $\tau_r(\mathcal W)$.  We now observe that $\frac{d}{d \alpha_i}L({\mathcal W}) = d\ell(\tau_r(\mathcal W))(v^{\otimes d}) \ne 0$ so $\mathcal W$ is not a critical point for $L$ and is thus escapable.
\end{proof}

Note that any $\mathcal W$ as in Lemma~\ref{lemma:low_rank_escapable} is such that $\tau_r(\mathcal W)$ has rank strictly less than $r$. An analogous statement, with an almost identical proof, holds for networks with ReLU or more general (positively) homogeneous activations.

\begin{proof}[Proof of Theorem~\ref{thm:escapable_n_2}] Any non-escapable point is a critical point so we consider $\mathcal W = (\alpha,W) \in \Crit(L)$. According to Lemma~\ref{lemma:singular_point}, there exists a non-zero $P \in \RR[X_1,X_2]_{d}$ such that $\nabla P(w_i) = 0$ for all $i$ such that $\alpha_i \ne 0$ and $w_i \ne 0$. As argued in the proof of Proposition~\ref{prop:branch_n_2}, this can only occur for $\lfloor \frac{d}{2} \rfloor = \lceil\frac{d+1}{2} \rceil - 1$ pairwise independent $w_i$. In particular, if $r \ge \lceil\frac{d+1}{2} \rceil $, then the conditions of Lemma~\ref{lemma:low_rank_escapable} are satisfied at $\mathcal W$ and we conclude that $\mathcal W$ is escapable.
\end{proof}

\begin{proof}[Proof of Theorem~\ref{thm:valleys_d_2}] 
We first remark that the proof of the (weaker) fact that all critical points are escapable
follows easily from \Cref{thm:fiber_quadratic} in the Appendix and \Cref{prop:branch_quadratic}. Indeed, any critical point that is not a global minimum must have rank less than
$n$. At these points, the fiber of the parameterization is connected. This implies that we find a path in the fiber to a parameter $\mathcal W$ that satisfies the conditions of \Cref{lemma:low_rank_escapable}.

This argument however does not exclude the existence of spurious valleys
without any local minima. To prove that no spurious valleys exist, we show the
following:
\begin{enumerate}
    \item \emph{Path lifting property:} If $S_1, S_2 \in \Sym^2(\RR^n)$ are
    such that all matrices in the path $\gamma(t) = (1-t) S_1 + t S_2$ have
    distinct eigenvalues, then for any $\mathcal W_1$ such that
    $\tau_r(\mathcal W_1) = S_1$ there exists a path $\mathcal W(t)$ in
    parameter space such that $\tau_r(\mathcal W(t)) = \gamma(t)$.
    \item For almost all pairs $S_1, S_2$ in $\Sym^2(\RR^n)$ the path
    $\gamma(t) = (1-t) S_1 + t S_2$ contains only matrices with distinct
    eigenvalues.
\end{enumerate}
Before arguing these two points we observe that they imply the claim. Indeed,
let $\mathcal W_{\min}$ be a global minimum of $L$ and let $H \subset \RR^r
\times \RR^{r \times n}$ be a connected component of a sublevel set of $L$
such that $\lim \inf_{\mathcal W \in H} L(\mathcal W) > L(\mathcal W_{\min})$
(so $H$ is a spurious valley). Let $V$ be a neighborhood of $\mathcal
W_{\min}$ such that $L(\mathcal W') < \lim \inf_{\mathcal W \in H} L(\mathcal
W)$ for all $\mathcal W'
\in V$. The image of $V$ and $H$ under $\tau_r$ are full measure sets in
$\Sym^2(\RR^n)$ (by the strong version of Sard's Theorem).\footnote{We can assume
without loss of generality that $H$ is open, by replacing a closed sublevel
set $L^{-1}([-\infty,t])$ with an open sublevel set
$L^{-1}([-\infty,t+\epsilon))$ for $\epsilon$ sufficiently small.} Using the
second point above, there exists $S_1 \in \tau_r(H)$ and $S_2 \in \tau_r(V)$
such that the path $\gamma(t) = (1-t) S_1 + t S_2$ does not contain matrices
with repeated eigenvalues. By 1, this path can be lifted to a path
$\mathcal W(t)$ such that $\mathcal W(0) \in H$, and $\tau_r(\mathcal W(t)) =
\gamma(t)$. Since $\ell$ is convex, $\ell(\gamma(t))$ is either monotone
decreasing between $0$ and $1$ (because $\ell(S_1)>\ell(S_2)$ by assumption),
or it achieves a global minimum at $t_0 \in (0,1)$, where the value is less
than $\ell(S_2)$. Either way, the lifted path $\mathcal W(t)$ for $t \in
[0,1]$ or $t \in [0, t_0]$ is a descent path from a point on $H$ to a point
where $L$ is less than $\lim \inf_{\mathcal W \in H} L(\mathcal W)$, which is
a contradiction.

The two points above follow from known properties of symmetric matrices and
their eigenvalues. The first point can be argued by noting that when eigenvalues are distinct,
both eigenvalues and eigenvectors are described locally by continuous functions (by the
implicit function theorem, see also \cite[Section 7]{alekseevsky1998}). The second point follows from the
fact that the set of (real) matrices with repeated eigenvalues is an algebraic set
of codimension two in $\Sym^2(\RR^n)$~\cite[Section 7.5]{sturmfels}. In particular, a generic line
will not meet this set.
\end{proof}

\paragraph{Bad minima for wide networks.}

We now observe that the loss landscape can have bad local minima for networks with \emph{arbitrarily large} width $r$. Since by Theorem~\ref{thm:venturi_spurious}, when $r \ge \dim \Sym^d(\RR^n)$ the landscape has no spurious valleys, we deduce that any bad local minimum must be escapable. The landscape in this case resembles the one shown in Figure~\ref{fig:landscapes},~(c). A similar example was given for the case $d = 2$ (quadratic activations) in~\cite{kazemipour2019no}.

\begin{example}\label{ex:bad_minima_exist} Let $d = 2k$ and choose $r\ge r_{\rm fill}$ {to} be arbitrarily large. Consider a function $L$ of the form
\[
L(\mathcal W) = \mathbb E_{x\sim \mathcal D} |f_{\mathcal W}(x) - g(x)|^2,
\]
where $\mathcal D$ is any data distribution of full support on $\RR^n$ and $g$ is a \emph{positive} polynomial function of degree $d$, that is, $g(x) > 0$ holds for all $x \in \RR^n$. Then any $\mathcal W = (\alpha, W)$ such that $W = 0$ and $\alpha_i < 0$ is a non-global local minima for $L$. 
Indeed, any $\tilde{\mathcal W}$ in a sufficiently small neighborhood of $\mathcal W$ will be such that $f_{\tilde{\mathcal W}}(x) \le 0$ for all $x$, so that
\[
L(\tilde {\mathcal W}) =  \mathbb E_{x\sim \mathcal D} |f_{\tilde{\mathcal W}}(x) - g(x)|^2 \ge \mathbb E_{x\sim \mathcal D} |g(x)|^2 = L(\mathcal W).
\]
This implies $\mathcal W$ is a local minimum. It is not a 
global minimum since our assumptions guarantee that 
$L(\mathcal W) > 0$ however there exists $\mathcal W'$ such 
that $f_{\mathcal W'} = g$ and $L(\mathcal W') = 0$.
\end{example}

Compared to~\cite{kazemipour2019no}, we also show that these bad minima are relevant for gradient-based optimization, since they have a basin of 
attraction of positive measure. 
We will use the following property of gradient flow for polynomial
networks, which follows from the homogeneity of the activation function. An
analogous result has been shown for ReLU networks~\cite{du2018algorithmic}.

\begin{proposition}\label{prop:invariants} Let $L: \RR^r \times \RR^{r \times n} \rightarrow \RR$ be a
function of the form $L = \ell \circ \tau_r$, where $\ell: \Sym^d(\mathbb{R}^n) \rightarrow \mathbb{R}$ is smooth.  If $\gamma(t)$ is an
integral curve for the negative gradient field of $L$, i.e., $\dot
\gamma(t) = - \nabla (L(\gamma(t)))$, then the quantities $\delta_i = \alpha_i^2 -
\frac{1}{d}\|w_i\|^2$, $i=1,\ldots,r$ (where $d$ is the activation degree)
remain constant for all $t$.
\end{proposition}

\begin{proof}
Letting $\gamma(t) = \left(\alpha(t), W(t) \right)$, the negative gradient field equations may be written
\begin{equation} \label{eq:grad_flow}
    \begin{cases} 
    \dot \alpha_i(t) = - \left\langle \nabla \ell, w_i^{\otimes d} \right\rangle \\
    \dot w_i(t) = - d \, \nabla \ell \cdot \left( \alpha_i w_i^{\otimes (d-1)} \right).
    \end{cases}
\end{equation}
for all $t$ and $i$, from the chain rule {(see equation~\eqref{eq:differential_parameterization})}. Here $\nabla \ell$ is short for $\nabla \ell\left( \tau_r(\alpha(t), W(t))\right)$ in tensorized form (so it {can be viewed as a tensor} in $\textup{Sym}^d(\mathbb{R}^n)$).  Also, $\cdot$ stands for contraction.  Eliminating $\nabla \ell$, 
\begin{equation} \label{eq:grad_eliminate}
    \dot \alpha_i(t) \alpha_i(t) = \frac{1}{d} \dot w_i(t)^T w_i(t)
\end{equation}
for all $t$ and $i$.  The result is immediate by 
anti-differentiating \eqref{eq:grad_eliminate}{.}
\end{proof}

\begin{corollary} \label{cor:positive-attraction-basin}
Assume $d$ {is} even and $r \ge r_{\rm fill}$ {is} arbitrarily large. Let $\ell: \Sym^d(\mathbb{R}^n) \rightarrow \mathbb{R}$ be a smooth convex function with an isolated global minimum at $T \in \Sym^d(\mathbb{R}^n)$ such that $g(X) = \langle T, X^{\otimes d} \rangle_F$ is a negative polynomial. In particular, $\ell$ could be of the form 
\[
\ell(S) = \mathbb E_{x \sim \mathcal D}|f(x) - g(x)|^2, \quad f(x) = \langle S, X^{\otimes d} \rangle_F,
\]
where $\mathcal D$ is a distribution of full support on $\RR^n$, similar to Example~\ref{ex:bad_minima_exist}. Fix $\mathcal W_0 = (\alpha_0, W_0) \in \RR^r \times \RR^{r \times n}$ such that $S_0 = \tau_r(\mathcal W_0)$ corresponds to a positive polynomial $f_0(X) = \langle S_0, X^{\otimes d} \rangle_F$, and $\alpha_i^2 > \frac{1}{d} \|w_i\|^2$ for all $i$ (note that such $\mathcal W_0$ can be chosen from an open set in $\RR^r \times \RR^{r \times n}$).
Then the negative gradient flow for $L = \ell \circ \tau_r$ initialized at $\mathcal W_0$ either diverges or converges to a non-global minimum.
\end{corollary}

\begin{proof} The target $T$ lies in the function space $\mathcal F_r = \Sym^d(\RR^n)$ but is not dynamically reachable. Indeed, it follows from Proposition~\ref{prop:invariants} that $\alpha_i > 0$ holds throughout the trajectory initialized at $\mathcal W_0$. Every polynomial in the trajectory is positive, so $\ell$ never reaches zero.
\end{proof}

\begin{remark} \Cref{cor:positive-attraction-basin} implies that for polynomial networks of even degree, gradient flow can fail to reach global optima for certain initializations, regardless of the network's width. However, with random initializations where the top coefficient $\alpha_i$ has an equal likelihood of being positive or negative, the condition $\alpha_i^2 > \frac{1}{d} \|w_i\|^2$ for all $i${---which imposes that all dyanmically reachable polynomials are positive---}becomes increasingly unlikely at initialization as the width $r$ increases. Therefore, although the optimization landscape is never trivial for wide networks, favorable initializations become more probable.
\end{remark}

We conclude this section with an example in which the the global minimum of $L = \ell \circ \tau_r$ is at infinity in parameter space, even though $\ell: \Sym^d(\RR^n)\rightarrow \RR$ has compact sublevel sets. This phenomenon is due to the fact that the set of low-rank tensors is not topologically closed for $d \ge 3$ (see \cite[Example 6.6]{comon2008}).

\begin{example}\label{ex:minimum_infinity} Consider a shallow polynomial network of degree $d > 2$ and width $r=2$:
\[
f_{\mathcal W}(x) = \alpha_1(w_1 \cdot x)^d + \alpha_2(w_2 \cdot x)^d.
\]
Consider the teacher model $g(x) = x_1^{d-1} x_2$ and let $L(\mathcal W) = \|f_{\mathcal W}(x) - g(x)\|^2$ for some norm $\|\cdot \|$ on $\Sym^d(\RR^n)$. It is possible to show that $f_{\mathcal W} \ne g$ holds for all parameters $\mathcal W \in \RR^2 \times \RR^{2 \times n}$. Indeed, $f_{\mathcal W}$ corresponds to a tensor of rank at most $2$, while $g(x)$ corresponds to a tensor of rank exactly $d$ (see Proposition 5.6 in~\cite{comon2008}). It follows that $L(\mathcal W)$ is always strictly positive. However, consider the following sequence of parameters $\mathcal W(\tau)$ for $\tau\rightarrow \infty$:
\[
    \alpha_1(\tau) = \frac{\tau}{(d-1)}, \quad  \alpha_2(\tau) = -\frac{\tau}{(d-1)}, \quad w_1 = \left(1,\frac{1}{\tau},\ldots,0\right), \quad w_2 = (1,0,\ldots,0).
\]
It holds that $L(\mathcal W(\tau)) \rightarrow 0$. Indeed, writing $\epsilon = \frac{1}{\tau}$, we have
\[
    f_{\mathcal W(\tau)}(x) = \frac{1}{(d-1)\epsilon}[\left(x_1+\epsilon x_2\right)^d - x_1^d ] \,\,\, \rightarrow \,\,\, x_1^{d-1} x_2.
\]
Note that the sequence of parameters $\mathcal W(\tau)$ diverges, even though the sequence of corresponding functions $f_{\mathcal W(t)}$ converges to $g$.
\end{example}

\section{Teacher-Student Problems}
\label{sec:teacher-student}

{\textbf{Section overview.} 
We study \emph{teacher–student} training problems, where a student network approximates a fixed teacher function via a loss over the function space. For shallow polynomial networks, this setting corresponds to low-rank tensor approximation under various norms, including \emph{distribution-induced norms} determined by the training data. Such norms capture how the data distribution shapes optimization. We examine their effect on the optimization landscape through moment tensors and analyze critical points using a geometric framework for distance functions to manifolds. This leads to the notion of \emph{data discriminants}, which reveal how changes in the data can induce qualitative transitions in the landscape.}

\vspace{0.8em}

We consider loss functions of the form
\begin{equation}\label{eq:teacher_student_general}
L(\mathcal W) = \|f_{\mathcal W} - f_{\mathcal V}\|^2, \qquad \mathcal W \in \RR^r \times \RR^{r \times n}, \,\,  \mathcal V \in \RR^s \times \RR^{s \times n},
\end{equation}
where $f_\mathcal{V}$ is the fixed ``teacher" and $f_\mathcal{W}$ is the trainable ``student." {Minimizing $L(\mathcal W)$} corresponds to approximating one shallow network by another, and when $s$ is smaller than $r$ is equivalent to a low-rank tensor approximation problem under a potentially nonstandard norm
{that may encode the training input distribution.}

\subsection{Functional norms}\label{subsec:functional-norms}

We begin by describing functional norms of interest.  
To fix notation, consider $\mathcal{W} = (\alpha, W) \in 
\mathbb{R}^r \times \mathbb{R}^{r \times n}$ and $\mathcal{V} 
= (\beta, V) \in \mathbb{R}^s \times \mathbb{R}^{s \times 
n}$, and let $\tau({\mathcal{W}}) = \sum_{i=1}^r \alpha_i 
w_i^{\otimes d}$ and $\tau({\mathcal{V}}) = \sum_{i=1}^s \beta_i 
v_i^{\otimes d}$ be tensors in 
$\textup{Sym}^d(\mathbb{R}^n)$ such that $f_{\mathcal{W}}(x) = 
\langle \tau(\mathcal W), x^{\otimes r} \rangle_{{F} }$ and $f_{\mathcal{V}}(x) = 
\langle \tau(\mathcal V), x^{\otimes r} \rangle_{{F} }$. We will focus our attention on the following two types of norms:

\medskip
\medskip

\noindent\textbf{1) Frobenius norm.}  As before, $\| f_{\mathcal{W}}  - 
f_{\mathcal{V}} \|_F^2 := \langle \tau(\mathcal W) - 
\tau(\mathcal V), \tau(\mathcal W) - \tau(\mathcal V) \rangle_F$, 
where
\begin{equation}
    \langle S, T \rangle_F \, := \sum_{i_1, \ldots, i_d =1}^n S_{i_1, \ldots, i_d} T_{i_1, \ldots, i_d} \quad \textup{for all } S, T \in \textup{Sym}^d(\mathbb{R}^n).
\end{equation}

\noindent\textbf{2) Distribution-induced norms.}
Let $\mathcal{D}$ be a probability distribution over $\mathbb{R}^{{n} }$ (decaying fast enough). Then 
$\|f_{\mathcal{W}} - f_{\mathcal{V}}\|_{\mathcal{D}}^2 \, := \langle f_{\mathcal{W}} - f_{\mathcal{V}}, f_{\mathcal{W}} - f_{\mathcal{V}} \rangle_{\mathcal{D}}$, where 
\begin{equation} \label{eq:product_distribution}
    \langle g, h \rangle_{\mathcal{D}} \, := \mathbb{E}_{x \sim \mathcal{D}}[g(x) h(x)] \quad \textup{for all } g, h : \mathbb{R}^d \rightarrow \mathbb{R}.
\end{equation}

A special class of distribution-induced norms arise from discrete distributions $\mathcal{D} = \frac{1}{N} \sum_{i=1}^N 
\delta_{x_i}$, where $\{x_1, \ldots, 
x_N\} \subset \mathbb{R}^n$. This is the setting of \textbf{empirical risk minimization}, where $\| f_{\mathcal{W}} - f_{\mathcal{V}} \|_{ERM}^2 := 
\langle f_{\mathcal{W}} - f_{\mathcal{V}}, f_{\mathcal{W}} - f_{\mathcal{V}} \rangle_{{ERM} }$ for
\begin{equation} \label{eq:product_ERM}
\langle g, h  \rangle_{ERM} \, := \frac{1}{N}\sum_{i=1}^N g(x_i)h(x_i)  \quad \textup{for all } g, h : \mathbb{R}^{{n}} \rightarrow \mathbb{R}.
\end{equation}
\begin{proposition} The ERM inner product~\eqref{eq:product_ERM} defines a non-degenerate inner product over $\Sym^{{d}}(\RR^n)$ if and only if there exists no non-zero $n$-variate polynomial of degree $d$ that vanishes on $x_1,\ldots,x_N$. This is generically true for $N \ge \binom{n+d-1}{d}$.
\end{proposition}
\begin{proof} Consider the linear map $E_X: \Sym^d(\RR^n) \rightarrow \RR^N$ that associates a polynomial $g$ with the vector $E_X(g) = (g(x_1),\ldots,g(x_N))$ of its evaluations on the input points. Since $\langle g, h \rangle_{ERM} = \frac{1}{N}(E_X (g)) \cdot (E_X(h))$, where $\cdot$ is the standard inner product in $\RR^N$, we deduce that the ERM product is non-degenerate if and only if $E_X$ is injective.
\end{proof}

Distribution-induced inner products can always be expressed in terms of the $2d$-th moments of $\mathcal{D}$.

\begin{proposition}\label{prop:distribution-inner-product}
Let 
$g,h$ be homogeneous degree-$d$ forms, written as
$g(x) = \langle S, x^{\otimes d} \rangle_F$ and 
$h(x) = \langle T, x^{\otimes d} \rangle_F$ for 
$S,T \in \mathrm{Sym}^d(\mathbb{R}^n)$. The distribution-induced product~\eqref{eq:product_distribution} can be expressed in terms of the 
$2d$-th moment tensor of $\mathcal{D}$ as
\[
    \langle g, h \rangle_{\mathcal{D}}
    = \langle S \otimes T, M_{\mathcal{D},2d} \rangle_F
    = \mathrm{vec}(S)^\top \, \mathrm{mat}(M_{\mathcal{D},2d}) \, \mathrm{vec}(T),
\]
where 
$M_{\mathcal{D}, 2d} := \mathbb{E}_{x \sim \mathcal{D}}[x^{\otimes 2d}] 
\in \mathrm{Sym}^{2d}(\mathbb{R}^n)$ is the $2d$-th moment tensor, 
and {$\mathrm{vec}$ and $\mathrm{mat}$ denote the vectorization and 
matricization operators with respect to lexicographic ordering of indices (mapping $\textup{Sym}^d(\mathbb{R}^n)$ to $\mathbb{R}^{n^d}$ and $\textup{Sym}^{2d}(\mathbb{R}^n)$ to $\mathbb{R}^{n^d \times n^d}$, respectively).}
\end{proposition}
\begin{proof}
Starting from the definition, we have
\begin{equation}\label{eq:moment_distribution}
\langle g, h \rangle_{\mathcal{D}}
= \mathbb{E}_{x \sim \mathcal{D}} \big[ \langle S, x^{\otimes d} \rangle_F
\langle T, x^{\otimes d} \rangle_F \big]
= \left\langle S \otimes T, \mathbb{E}_{x \sim \mathcal{D}}\, x^{\otimes 2d} \right\rangle_F,
\end{equation}
which is precisely $\langle S \otimes T, M_{\mathcal{D},2d} \rangle_F$. 
The matricized form follows by applying $\mathrm{vec}$ to $S$ and $T$. 
\end{proof}

 Note that the representation of distribution-induced inner product in terms of moments makes no distinction between discrete and continuous distributions; in {the} latter case, we simply consider discrete moments.

It is natural to ask whether there exists a distribution that would make a distribution-induced norm coincide with the Frobenius norm. It is easy to see that this is not the case.
\begin{proposition}
Assuming $d \ge 2$, then $\langle \cdot, \cdot\rangle_{\mathcal D} \ne \langle \cdot, \cdot\rangle_{F}$ holds for all distributions $\mathcal D$.
\end{proposition} 
\begin{proof} We can write~\eqref{eq:moment_distribution} as
\[
\langle g, h\rangle_{\mathcal D} = \sum_{\substack{(i_1,\ldots,i_d) \in [n]^d\\(j_1,\ldots,j_d) \in [n]^d}} \mathbb E_{x \sim \mathcal D}[x_{i_1}\ldots x_{i_d} x_{j_1} \ldots x_{j_d}] s_{i_1,\ldots,i_d} t_{j_1,\ldots,j_d},
\]
where $[n]:=\{1,\ldots,n\}$ and $S = (s_{i_1,\ldots,i_d})$, $T = (t_{i_1,\ldots,i_d})$ and $x = (x_i)$. From this it follows that $\textup{mat}({M}_{\mathcal{D},2d})$ is never a diagonal matrix {(for example, consider the entries where $(i_1,\ldots,i_d) = (j_1,\ldots,j_d)$)}. Thus, $\langle g, h \rangle_{\mathcal D} \ne {\rm vec}(S)^T {\rm vec}(T)$ for all $\mathcal D$.
\end{proof}

We next provide explicit formulas for the even moments of various data distributions.

\medskip

\noindent\textbf{Rotationally invariant $\mathcal{D}$.}  Assume that the random vector $x \in \mathbb{R}^n$ factors as $x = \rho y$ where $\rho \in \mathbb{R}_{\geqslant 0}$ is a random radius and $y \in S^{n-1}$ is a uniformly random direction vector on the unit sphere.  Equivalently, $\mathcal{D} = \mathcal{D}_{\rho} \otimes \mathcal{U}(S^{n-1})$, where $\mathcal{D}_{\rho}$ is some probability distribution on $\mathbb{R}_{\geqslant 0}$ and $\mathcal{U}(S^{n-1})$ is the uniform probability distribution on $S^{n-1}$. Then,
        \begin{equation}
            {M}_{\mathcal{D},2d} = \mathbb{E}_{x \sim \mathcal{D}} [x^{\otimes 2d}] = \mathbb{E}_{\substack{\rho \sim \mathcal{D}_{\rho} \\ y \sim \mathcal{U}}} [(\rho y)^{\otimes 2d}] = \mathbb{E}_{\rho \sim \mathcal{D}_{\rho}} [\rho^{2d}] \,  \mathbb{E}_{y \sim \mathcal{U}}[y^{\otimes 2d}].
        \end{equation}
        This shows ${M}_{\mathcal{D}, 2d}$ is the same for all rotationally invariant $\mathcal{D}$, up to non-negative scale.  In particular, it suffices to consider $x \sim \mathcal{N}(0, I)$.  For this, we fix $z \in \mathbb{R}^n$ arbitrarily and compute
        \begin{equation}
        \small    \left\langle \mathbb{E}_{x \sim \mathcal{N}(0,I)}[x^{\otimes 2d}], z^{\otimes 2d} \right\rangle_F =  \mathbb{E}_{x \sim \mathcal{N}(0,I)}[(z^T x)^{2d}] = \mathbb{E}_{\lambda \sim \mathcal{N}(0, \|z\|_2^2)}[\lambda^{2d}] = \|z\|_2^{2d} \mathbb{E}_{\tilde{\lambda} \sim \mathcal{N}(0,1)}[\tilde{\lambda}^{2d}].
        \end{equation}
        On the other hand, the RHS equals $\langle 
        \textup{sym}(I^{\otimes d}), z^{\otimes 2d} \rangle$, 
        up to a scale independent of $z$ (note that $I^{\otimes d}$ is not a symmetric tensor in $(\RR^n)^{\otimes 2d}$; see the example below).  It follows  ${M}_{\mathcal{D}, 2d} \propto {M}_{\mathcal{N}(0,I), 2d} \propto 
        \textup{sym}(I^{\otimes d})$, and the scalings are 
        non-negative.

\medskip

\noindent{\bf $\mathcal{D}$ with i.i.d.\
        coordinates.} Here $x \in \RR^{{n}}$, 
        with i.i.d.\ coordinates according to some 
        probability distribution $\pi$ on $\mathbb{R}$.  
        Equivalently, $\mathcal{D} = \pi^{\otimes {n}}$.
        Now for each $s \geq 0$, denote by $\mu_{\pi, s}$ the $s$-th moment of $\pi$, \ie,  $\mu_{\pi, s} = \mathbb{E}_{y \sim \pi}  [y^{s}] \in \mathbb{R}$.  Also, given  $I \in [n]^{2d}$, we let $\#(I, s)$ be the number of elements of $[n]$ that occur exactly $s$ times in $I$.  Then,
        \begin{equation}
            \left({M}_{\mathcal{D},2d}\right)_{I} = \left( \mathbb{E}_{x \sim \mathcal{D}}[x^{\otimes 2d}] \right)_{I} = \Pi_{s=0}^{2d} \, \mu_{\pi,s}^{\#({I}, s)}.
        \end{equation}
        Moments of  distributions with independent but not necessarily identical coordinates are derived in \cite{zhang2023moment,alexandr2023moment}.

\noindent{\bf Centered, colored Gaussians $\mathcal{D}$.}  Here $x \sim \mathcal{N}(0, \Sigma)$ for some covariance matrix $\Sigma \in \mathbb{R}^{n \times n}$.
Reasoning as in the case of rotationally invariant $\mathcal{D}$ shows $M_{\mathcal{N}(0,\Sigma),2d} \propto \textup{sym}(\Sigma^{\otimes d})$. Indeed, let $\Sigma = AA^T$ be a Cholesky factorization with $A \in \mathbb{R}^{n \times n}$, so that $x = Ay$ where $y \sim \mathcal{N}(0,I)$.  Fix $z \in \mathbb{R}^n$ arbitrarily and compute
\begin{equation}
    \left\langle \mathbb{E}_{x \sim \mathcal{N}(0,\Sigma)}[x^{\otimes 2d}], z^{\otimes 2d} \right\rangle_F =  \mathbb{E}_{y \sim \mathcal{N}(0,I)}[(z^T Ay)^{2d}] = \mathbb{E}_{\lambda \sim \mathcal{N}(0, \|A^T z\|_2^2)}[\lambda^{2d}].
\end{equation}
The RHS is proportional to $\| A^T z \|_2^{2d} = \langle AA^T, zz^T \rangle_F^{d} = \langle \textup{{sym}}(\Sigma^{\otimes d}), z^{\otimes 2d} \rangle_F$, as desired. 
Formulas for the moment tensors of non-centered Gaussians may be found in \cite{pereira2022tensor}.

\medskip
        
\noindent{\bf Mixtures $\mathcal{D}$}.  Let $\mathcal{D} = \sum_{i=1}^{t} \gamma_i \mathcal{D}_i$ be a convex combination of  distributions $\mathcal{D}_i$. Then,
\begin{equation}
    {M}_{\mathcal{D},2d} = \mathbb{E}_{x \sim \mathcal{D}}[x^{\otimes 2d}] = \sum_{i=1}^t \gamma_i \mathbb{E}_{x \sim \mathcal{D}_i}[x^{\otimes 2d}] = \sum_{i=1}^t \gamma_i {M}_{\mathcal{D}_i, 2d}.
\end{equation}
Many properties of the moments of various mixture models are  developed in  \cite{pereira2022tensor,zhang2023moment,alexandr2023moment}.

\begin{example} Assume $n=d=2$. A distribution-induced inner product on  the space {of}  $2 \times 2$ matrices is described by a $4 \times 4$ matrix of $4$-th moments:
\[
{\rm mat}(M_{\mathcal D, 4}) =
        \begin{pmatrix}
          \colorbox{yellow!30}{$\mathbb E[x_1^4]$} & \colorbox{red!30}{$\mathbb E[x_1^3x_2]$}& \colorbox{red!30}{$\mathbb E[x_1^3x_2]$}&\colorbox{green!30}{$\mathbb E[x_1^2x_2^2]$}  \\
         \colorbox{red!30}{$\mathbb E[x_1^3x_2]$} &\colorbox{green!30}{$\mathbb E[x_1^2x_2^2]$} & \colorbox{green!30}{$\mathbb E[x_1^2x_2^2]$}& \colorbox{orange!30}{$\mathbb E[x_1x_2^3]$}\\
         \colorbox{red!30}{$\mathbb E[x_1^3x_2]$} &\colorbox{green!30}{$\mathbb E[x_1^2x_2^2]$} &\colorbox{green!30}{$\mathbb E[x_1^2x_2^2]$} & \colorbox{orange!30}{$\mathbb E[x_1x_2^3]$}\\
          \colorbox{green!30}{$\mathbb E[x_1^2x_2^2]$}&\colorbox{orange!30}{$\mathbb E[x_1x_2^3]$} & \colorbox{orange!30}{$\mathbb E[x_1x_2^3]$}& \colorbox{pink!30}{$\mathbb E[x_2^4]$}
        \end{pmatrix},
\]
where we use colors to highlight identical entries. If $\mathcal D$ is a standard Gaussian distribution, then this specializes to
\begin{align*}
{\rm mat}(M_{\mathcal N(0,I), 4})=\begin{pmatrix}
        3 & 0 & 0 &1\\
        0 & 1 & 1 & 0\\
        0 & 1 & 1 &0\\
        1 & 0 & 0 & 3
    \end{pmatrix}.
\end{align*}
We observe that this matrix is proportional to the matricization of ${\rm sym}(I^{\otimes 2})$, where each entry with index $(i,j,k,l)$ is given by $\delta_{ij}\delta_{kl} + \delta_{ik}\delta_{jl} + \delta_{il} \delta_{jk}$.
In the next section, we also give a coordinate-free expression for distribution-induced inner products when $n=2$ and $\mathcal D$ has i.i.d. and centered coordinates (\Cref{prop:iid-dist}).
\end{example}

\subsection{Critical points of the distance function}

Let us now take a step back and describe some general aspects of the distance function of points to smooth manifolds that will be used to study the landscape of teacher-student problems for shallow polynomial networks.  Due to the generality of the setup, results in this subsection may also be of independent interest.

Let $X \subseteq \RR^n$ be a manifold or variety of dimension $d$.  
For example, we may think of $X$ as the function space of feasible student models. Let $\PD_n$ denote real $n \times n$  positive-definite matrices.
For each $t \in \RR^n$ and $\Sigma \in \PD_n$, consider the optimization problem
\begin{equation} \label{eq:most-general-problem}
    \min_{s \in X} L_{t, \Sigma}(s), \qquad  L_{t, \Sigma}(s) := \| s - t \|_{\Sigma}^2,
\end{equation}
where $\|v\|_\Sigma^2 := v^\top \Sigma v$.
Problem \eqref{eq:most-general-problem} seeks the nearest point on $X$ to $t$, when distances are measured according to $\Sigma$.  As before, we call $s$ and $t$ the student and teacher, respectively. 
\par
For fixed $t, \Sigma$, the first-order critical points of \eqref{eq:most-general-problem} are 
\begin{equation}
\Crit (L_{t, \Sigma}|_{X}) := \{ s \in {X_{\rm reg}} : d L_{t,\Sigma}(T_{s}X)=0 \},
\end{equation}
{where $X_{\rm reg}$ denotes the set of smooth points on $X$, and $dL_{t,\Sigma}(T_s X) = 0$ means that the differential of $L_{t,\Sigma}$ at $s$ vanishes in all directions tangent to $X$.} Since almost all functions  are Morse functions, $\Crit (L_{t,\Sigma}|_{X})$ is typically a finite set. Using local parameterizations, we can also associate a Hessian signature to each critical point $s \in \Crit (L_{t,\Sigma}|_{X})$.  More precisely, fix a chart $p : U \rightarrow V$, where $U \subseteq \RR^d$ and $V \subseteq X$ are open neighborhoods respectively of $0$ and $s$, such that $p(0) = s$; then the signature and rank of $H(L\circ p)(0)$ where $H$ denotes the (usual) Hessian matrix does not depend on the choice of chart. We sometimes use $H_{L,X}$ to denote any Hessian matrix associated with a chart, focusing solely on its rank or signature.
\par
We consider the size of $\Crit (L_{t,\Sigma}|_{X})$ together with the Hessian signature of all of its elements as the ``qualitative information'' characterizing the critical points for Problem~\eqref{eq:most-general-problem}. As $t$, $\Sigma$ vary this qualitative information will change. To describe how $t$ and $\Sigma$ affect this change, we introduce the sets 
\[
\begin{aligned}
&Z_X:= \{(s,t,\Sigma) \in X \times \RR^n \times \PD_n \colon s \in \Crit (L_{t, \Sigma}|_{X}) \text{ s.t. } \det H_{L_{t,\Sigma}}(s) = 0\},\\[.2cm]
&F_X :=
\{(t,\Sigma) \in \RR^n \times \PD_n \colon \exists s \in X \text{ s.t. } (s,t,\Sigma) \in Z_X \},\\[.2cm]
&F_{X,\Sigma} := \{t \in \RR^n \colon (t,\Sigma) \in F_X\}.
\end{aligned}
\]
The set $Z_X$ consists of triples $(s,t,\Sigma)$ of student, teacher, and inner product, such that $s$ is a degenerate critical point for $L_{t, \Sigma}|_{X}$. The other sets are derived from $Z_X$. Classically, $F_{X,\Sigma}$ is known as the \emph{focal locus} of $X$ (typically with {$\Sigma$ the identity}). When $X$ is an algebraic variety, its Zariski closure is also called the \emph{ED discriminant}~\cite{draisma2013}.  A point $t \in \RR^n$ is called a \emph{focal point} of $(X,s)$ of multiplicity $m\ge 1$ for the distance $\Sigma$ if $(s,t,\Sigma) \in Z_{X}$ and the Hessian $H_{L_{t,\Sigma}}(s)$ has nullity $m$. The following result is adapted from~\cite[Chapter 6]{milnor2016morse}.

\begin{proposition}\label{prop:hessian-differential} Consider a fixed $\Sigma \in \PD_n$ and let
$N_{X,\Sigma} := \{(s,v) \in X \times \RR^n \colon \langle v, w\rangle_\Sigma = 0\,\, \forall w \in T_s X\}$ be the total space of the normal bundle to $X$. A point $t \in \RR^n$ is a focal point of $(X,s)$ of multiplicity $m$ if and only if the differential of the ``endpoint map'' $e_{X,\Sigma}: N_{X,\Sigma} \rightarrow \RR^n, (s,v) \mapsto s+v$ has nullity $m$ at $(s,t-s)$. In particular, $F_{X,\Sigma} = {\rm Br}(e_{X,\Sigma})$, where {\rm Br} denotes the branch locus (\Cref{subsec:functional-space}).
\end{proposition}
\begin{proof} Without loss of generality we assume $\Sigma = {\rm Id}_n$. Let $x(u) \in X$ be a local parameterization of $X$ with $u \in \RR^d$ and consider $n-d$ vector fields $w_1(u),\ldots,w_{n-d}(u)$ in $\RR^n$ that form an orthonormal frame of the normal bundle to $X$. We parameterize $N_{X,\Sigma}$ as $(s,v) = \left(x(u), \sum_{m=1}^{n-d} \gamma_m w_m(u)\right)$ with coordinates $(u,\gamma) \in \RR^d \times \RR^{n-d}$. If $t = s+v$, we have that
\begin{equation}\label{eq:partial-derivatives}
\left\{
\begin{aligned}
&\frac{\partial t}{\partial u_i} = \frac{\partial x}{\partial u_i} + \sum_{m=1}^{n-d} \gamma_{{m}} \frac{\partial w_m}{\partial u_{{i}}}, \qquad &&i=1,\ldots d,\\[.2cm]
& \frac{\partial t}{\partial \gamma_l} = w_l, &&l=1,\ldots,n-d.
\end{aligned}
\right.
\end{equation}

Taking the inner product of these $n$ vectors with the linearly independent vectors $\frac{\partial x}{\partial u_1},\ldots,\frac{\partial x}{\partial u_{d}},$ $w_1,\ldots,w_{n-d}$ we obtain an $(n\times n)$ matrix as follows:
\begin{equation}
\setlength\arraycolsep{12pt}
\begin{pmatrix}
\frac{\partial x}{\partial u_i} \cdot \frac{\partial x}{\partial u_j} + \sum_{m=1}^{n-d} \gamma_m \frac{\partial w_m}{\partial u_i} \cdot \frac{\partial x}{\partial u_j} & \sum_{m=1}^{n-d} \gamma_m \frac{\partial w_m}{\partial u_i} \cdot w_l\\[.4cm]
\bf{0} & {\rm Id}_{n-d}
\end{pmatrix}.
\end{equation}
The span of the vectors in~\eqref{eq:partial-derivatives} has dimension equal to the rank of this matrix, so it is equal to the rank of the upper-left block. We now observe that
\[
0 = \frac{\partial}{\partial u_i} \left(w_m \cdot \frac{\partial x}{\partial u_j}\right)= \frac{\partial w_m}{\partial u_i} \cdot \frac{\partial x}{\partial u_j} + w_m \cdot \frac{\partial^2 x}{\partial u_i \partial u_j},
\]
so the rank of the upper-left block is the rank of
\begin{equation}\label{eq:fundamental forms}
\frac{\partial x}{\partial u_i} \cdot \frac{\partial x}{\partial u_j} - \sum_{m=1}^{n-d} \gamma_m \, w_m \cdot \frac{\partial^2 x}{\partial u_i \partial u_j}.
\end{equation}
Finally, if $L(u)= \|x(u) - t\|^2$, we have that
\begin{equation}\label{eq:hessian}
\frac{\partial^2 L}{\partial u_i \partial u_j} = 2 \frac{\partial}{\partial u_i} \left(\frac{\partial x}{\partial u_j} \cdot (x - t)\right) = 2 \left(\frac{\partial x}{\partial u_i}\cdot \frac{\partial x}{\partial u_{{j}}} - \frac{\partial^2 x}{\partial u_i \partial u_j} \cdot (t - x) \right).
\end{equation}
Thus, if $t = x(u) + \sum_{m=1}^{n-d} \gamma_m w_m(u)$, then~\eqref{eq:fundamental forms} and~\eqref{eq:hessian} are the same up to the factor of $2$, and the nullity of the differential of the endpoint map $e_{X,\Sigma}$ at $(s,t-s)$ is the same as that of the Hessian $H_{L_{t,\Sigma}}(s)$.
\end{proof}

\begin{remark}\label{rmk:principal-curvatures} As shown in the proof of Proposition~\ref{prop:hessian-differential}, if $t = s + \gamma v$ where $v$ is a unit vector orthogonal to $X$, then the Hessian of the distance function $H_{L_{t,\Sigma}}(s)$ can be written as $G - \gamma (v \cdot E)$ with $G_{ij} := \frac{\partial x}{\partial u_i}\cdot \frac{\partial x}{\partial u_{{j}}}$ and $(v \cdot E)_{ij} := v \cdot \frac{\partial^2 x}{\partial u_i \partial u_j}$. The bilinear forms $G$ and $(v \cdot E)$ are called respectively the \emph{first fundamental form} and the \emph{second fundamental form in the direction of $v$}. If local coordinates are chosen so that $G = {\rm I}_d$, then the eigenvalues of $(v \cdot E)$ are called the \emph{principal curvatures} of $X$ at $s$ in the normal direction of $v$. In this setting, it is easy to see that $t$ is a focal point if and only if $\gamma = 1/k$ where $k$ is a principal curvature.
\end{remark}

\begin{theorem} \label{thm:disc-1} Let $X \subset \RR^n$ be a smooth manifold, and assume that $X$ is either compact or is a cone in $\RR^n$ without the vertex. If $(t, \Sigma) \in (\RR^n \times \PD_n) \setminus F_X$, then there exists an open neighborhood $\mathcal{U} \subseteq \RR^n \times \PD_n$ of $(t, \Sigma)$ such that:
\begin{enumerate}
    \item All critical points for $(t', \Sigma') \in \mathcal U$ can be locally parameterized by smooth functions. That is, there exist smooth functions $s_1, \ldots, s_k : \mathcal{U} \rightarrow X$ such that $\Crit (L_{t', \Sigma'}|_{X}) = \left\{s_1(t', \Sigma'), \ldots, s_k(t', \Sigma')\right\}$ for all $(t', \Sigma') \in \mathcal{U}$.
    \item For each critical point $s_i$ considered above, the Hessian $H_{L_{t',\Sigma}}(s_i)$ is non-degenerate and its signature is constant over $(t',\Sigma') \in \mathcal U$.
\end{enumerate}
\end{theorem}

\begin{proof}
\emph{i)} Consider an arbitrary critical point $s \in \Crit (L_{t, \Sigma}|_{X})$. According to Proposition~\ref{prop:hessian-differential}, the fact that $(t, \Sigma) \not \in F_X$ implies that the endpoint map $e_{X,\Sigma}:(s,v) \mapsto s+v$ is a local diffeomorphism at $(s,t-s)$. Moreover, the same is true for the ``extended'' endpoint map $e_{X}: (s,v,\Sigma) \mapsto (s+v,\Sigma)$, since the last component is the identity. It now follows from the Inverse Function Theorem that $e_X$ can be locally inverted in a neighborhood of $(s,t-s,\Sigma)$, so $s$ can be smoothly parameterized by $t$ and $\Sigma$. Finally, we observe that if $X$ is compact or a cone, then there exists a neighborhood $\mathcal U$ of $(t,\Sigma)$ such that the number of critical points $\Crit (L_{t', \Sigma'}|_{X})$ is locally constant for $(t',\Sigma') \in \mathcal U$, and {all} critical points for $(t',\Sigma') \in \mathcal U$ can be smoothly parameterized by $(t', \Sigma')$. To see this, we note that under our assumptions we can find a neighborhood $\mathcal V$ of $(t,\Sigma)$ such that $e_X^{-1}(\mathcal V) \subset\mathcal K$ where $\mathcal K$ is compact (in other words, critical students are bounded): this is clear when $X$ is compact, while if $X$ is a cone it is true since $s \in \Crit (L_{t, \Sigma}|_{X})$ implies $\langle t-s,s\rangle_\Sigma = 0$ so $\|s\|_\Sigma ^2\le \|s\|_\Sigma ^2 + \|t-s\|^2_\Sigma =  \|t\|^2_\Sigma$. Now since $(t,\Sigma)$ is a regular value of the extended endpoint map, $e_{X}^{-1}(t,\Sigma) = \{(s_1,v_1,\Sigma),\ldots,(s_k,v_k,\Sigma)\}$ is a finite set (it is discrete and contained in $\mathcal K$). If we consider pairwise disjoint neighborhoods $U_1,\ldots,U_k \subset N_{X}$ of $(s_i, v_i, \Sigma)$ that are mapped diffeomorphically to $V_1,\ldots,V_k \subset \RR^n \times \PD_n$, we can then let $\mathcal U = V_1 \cap \ldots \cap V_k \setminus e_{X}(\mathcal K \setminus (U_1 \cup \ldots \cup U_k))$.
\par
\emph{ii)} Since by construction, $\mathcal U \subset (\RR^n \times \PD_n) \setminus F_X$ we have $\det H_{L_{t',\Sigma}}(s_i) \ne 0$ for all critical students $s_i$, so the result follows the from {the} continuity of the eigenvalues of the Hessian.
\end{proof}

{This result provides a tool to study the qualitative behavior of the optimization landscape as the data varies.} Indeed,
when $X$ is (the smooth locus of) a nonlinear algebraic variety, the complex Zariski closure of the extended focal locus $F_X$ is a hypersurface~\cite[Theorem 1]{catanese2000focal}. In particular, there exists a \emph{single polynomial equation} $P(t,\Sigma) = 0$ in both the teacher and inner product that must be satisfied when there is a transition in the qualitative properties of the landscape (equivalently, when this equation is not satisfied, the qualitative landscape remains locally unchanged). Moreover, if the inner product $\Sigma = \Sigma_{\mathcal D}$ is induced by a distribution $\mathcal D$ as above, then we also obtain a polynomial $Q(t, {\bm \mu}):=P(t, \Sigma_{\mathcal D})$ in the teacher and moments $\bm \mu$ of $\mathcal D$. The equation $Q(t, {\bm \mu})=0$ then precisely defines all ``unstable data'' for the optimization landscape. We refer to $P$ and $Q$ as the \emph{teacher-metric discriminant} and the \emph{teacher-data discriminant}, respectively. {More broadly, both are instances of what we call a \emph{data discriminant}: loci in data spaces that reflect qualitative changes in the optimization of the training loss.}
We also note that \Cref{thm:disc-1} can in some cases be applied to nonsmooth function spaces stratum-by-stratum, as we do in Section 4 for determinantal varieties which are stratified by matrix rank.

The next result from~\cite{milnor2016morse} shows that the signature of the Hessian of the distance function has an intuitive geometric description in terms of the focal locus. We will use this fact to easily compute the signature of the Hessian of the loss function for quadratic networks {in \Cref{sec:quadratic}}.

\begin{theorem}\label{thm:focal_points}
For any $t \in \RR^n \setminus F_{X, \Sigma}$ and $s \in \Crit(L_{t,\Sigma}|_{X})$, the \emph{index} of $L_{t,\Sigma}|_{X}$ at $s$ (\ie, number of negative eigenvalues of the Hessian $H_{L_{t,\Sigma}}$) is equal to the number of focal points of $(X,s)$ which lie in the segment {from} $s$ {to} $t$, each counted with multiplicity.
\end{theorem}

\begin{proof} In the proof of \Cref{prop:hessian-differential} we noted that $\frac{\partial^2 L}{\partial u_i \partial u_j} =  2 \left(\frac{\partial x}{\partial u_i}\cdot \frac{\partial x}{\partial u_{{j}}} - \frac{\partial^2 x}{\partial u_i \partial u_j} \cdot (t - s) \right)$, where $x(u)$ is a local parameterization of $X$. 
Without loss of generality we can assume that $\frac{\partial x}{\partial u_i}\cdot \frac{\partial x}{\partial u_{{j}}}$ is the identity. If $t-s = \gamma w$, with $\gamma \in \RR$ and $\|w\|=1$, then the index of the Hessian is the number of eigenvalues $\frac{\partial^2 x}{\partial u_i \partial u_j} \cdot w$ (or principal curvatures, cf. Remark~\ref{rmk:principal-curvatures}) which are greater than $1/\gamma$. On the other hand, $t$ is a focal point if and only if $1/\gamma$ is an eigenvalue, and its multiplicity is equal to the multiplicity of $1/\gamma$ as an eigenvalue.
\end{proof}

\subsection{Examples}
\label{subsec:examples}

We briefly present some computations of discriminants performed using Macaulay2.\footnote{Further details, including the code and outputs of these computations, can be found at~\url{https://github.com/ElaPolak/Optimization-Landscape-of-Teacher-Student-Problems}.}

\begin{figure}[t]
    \centering
    \begin{minipage}[c]{0.33\textwidth}
        \centering
        \includegraphics[width=\textwidth]{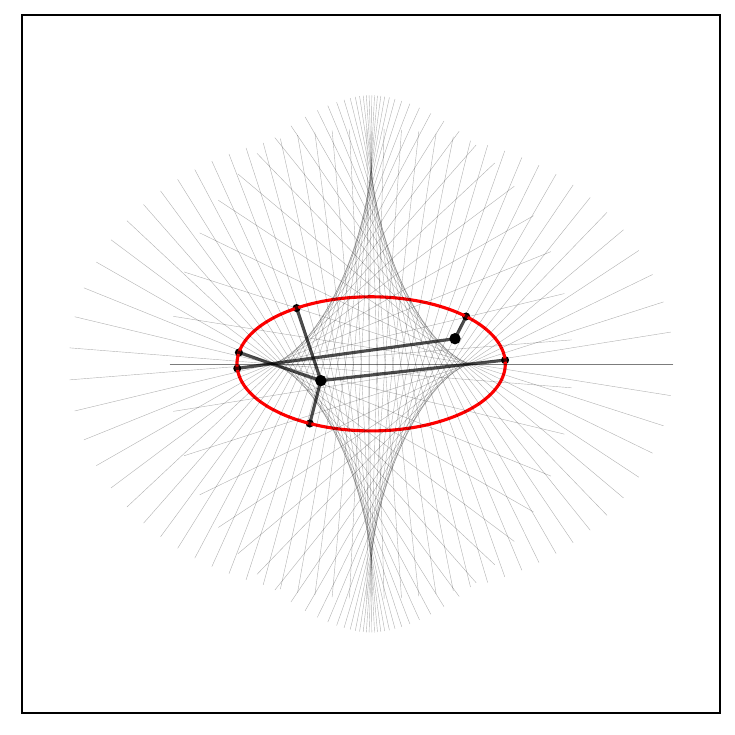} 
        \label{fig:fig1}
    \end{minipage}
    \hfill
    \begin{minipage}[c]{0.26\textwidth}
        \centering
        \includegraphics[width=\textwidth]{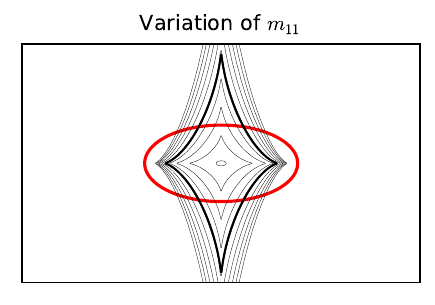}
        \label{fig:fig2a}
        \vspace{-1.5em} 
        
        \includegraphics[width=\textwidth]{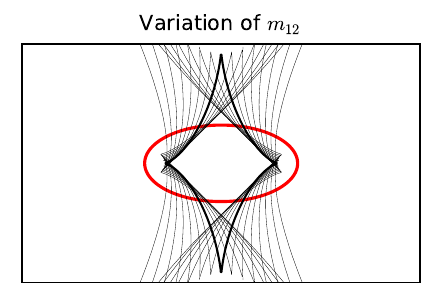} 
        \label{fig:fig2b}
    \end{minipage}
    \hfill
    \begin{minipage}[c]{0.35\textwidth}
        \centering
        \includegraphics[width=\textwidth]{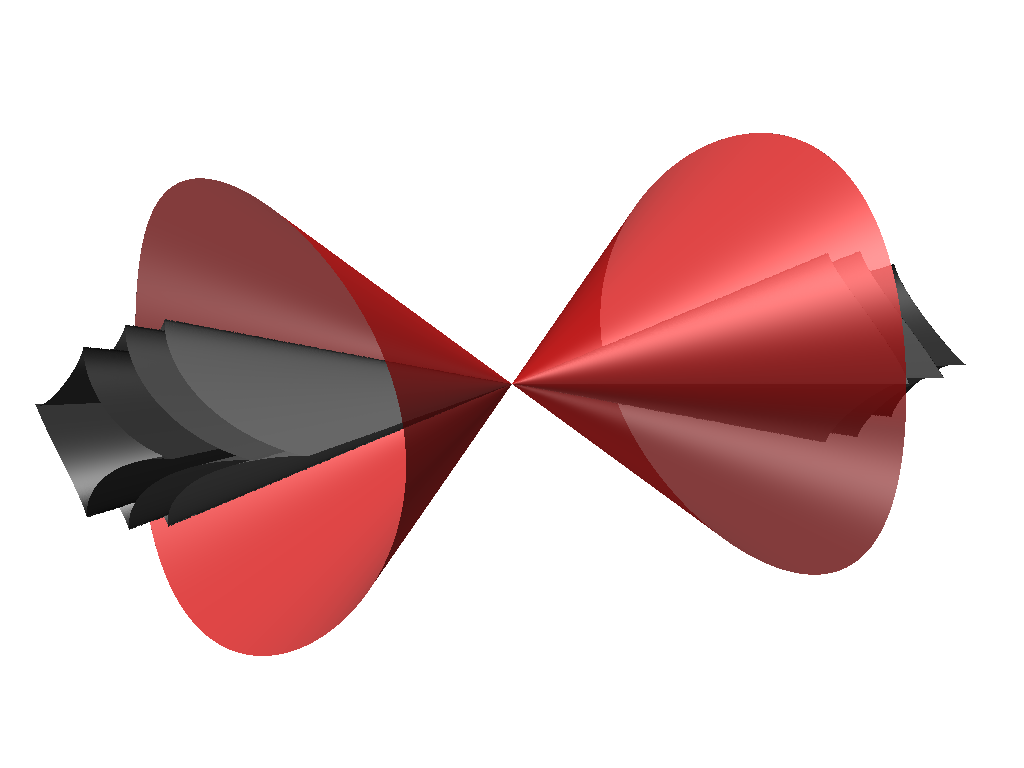}
        \label{fig:fig3}
    \end{minipage}
    \caption{Visualizations of discriminants. \emph{Left:} The focal locus of an ellipse with two 'teacher' points and the corresponding critical points for the associated teacher-student problem. \emph{Center:} Focal curves encoded by the teacher-metric discriminant $P(x,y,m_{11}, m_{12}, m_{22}$). The top figure illustrates curves for $m_{22}=1, m_{01} = 0$ and $m_{11} = t$ (varying $t$) , while the bottom figure  for $m_{11}=m_{22} = 1$ and $m_{10} = t$ (varying $t$); \emph{Right:} A 3D plot of the determinantal variety $X = \{(t_{00}, t_{01}, t_{11}) \colon t_{00}t_{11} - t_{01}^2 = 0$\} with three distinct focal surfaces corresponding to $m_{0000}=m_{1111}=1$, $m_{0001}=m_{0011}=m_{0111}=0$ and $m_{0101}=1,2,4$ (three surfaces).}
    \label{fig:discriminant-examples}
\end{figure}

\begin{example}
Let $X = \{(x,y) \colon x^2/4 + y^2/1 = 1\} \subset \mathbb R^2$ be an ellipse. The classical focal locus is a diamond-shaped curve that is the envelope of all normal lines, and bounds regions where the teacher-student problem has 4 or 2 (real) critical points (\Cref{fig:discriminant-examples}, left). Varying the inner product $\Sigma = \begin{bmatrix} m_{11} & m_{12} \\ m_{12} & m_{22} \end{bmatrix} \in {\rm PD}_2$, this curve will change and we can describe all curves using a single polynomial equation $P(x,y, m_{11},m_{12}, m_{22})=0$ which characterizes unstable teacher-metric pairs. The polynomial $P$ has degree 12 (6 and 6 in the two groups of variables) and 152 terms.
\end{example}

\begin{example}\label{ex:2x2-discriminant}
Let $X$ be the set of $2 \times 2$ symmetric matrices of rank at most $1$. We parameterize symmetric matrices as $T = \begin{bmatrix} t_{00} & t_{01} \\ t_{01} & t_{11} \end{bmatrix}$ and inner products on this space as:
$$
\Sigma = \begin{bmatrix}
m_{0000} & m_{0001} & m_{0001} & m_{0011} \\
m_{0001} & m_{0101} & m_{0101} & m_{0011} \\
m_{0001} & m_{0101} & m_{0101} & m_{0011} \\
m_{0011} & m_{0111} & m_{0111} & m_{1111}
\end{bmatrix}.
$$
In this setting, the teacher-metric discriminant is a homogeneous polynomial \[P(t_{00}, t_{01}, t_{11}, m_{0000}, m_{0001}, m_{0011}, m_{0101}, m_{0111}, m_{1111})\] of bidegree $(6, 12)$ with $5938$ terms. The Frobenius inner product corresponds to $m_{0000} = m_{1111} = 1$, $m_{0101} = \frac{1}{2}$, and $m_{0001} = m_{0011} = m_{0111} = 0$. Under these conditions, the discriminant specializes to
$$
(t_{00}^2 + 4 t_{01}^2 - 2 t_{00} t_{11} + t_{11}^2)^3,
$$
which is (a power of) the discriminant of the characteristic polynomial of $T$, as noted in~\cite{draisma2013}. For distribution-induced inner products, $m_{ijkl} = \mathbb{E}[x_i x_j x_k x_l]$, and in particular $m_{0011} = m_{0101}$. Imposing this condition on $P$ reduces it to a polynomial with $2060$ terms. If we further assume a symmetric i.i.d. distribution and define $\mu_4 = m_{0000} = \mathbb{E}[x^4]$ and $\mu_2^2 = m_{0011} = m_{0111} = \mathbb{E}[x^2]^2$, {the discriminant specializes to a polynomial beginning as
\begin{align*}
&729 \mu_2^{24} t_{00}^{3} t_{11}^{3} 
+ 1458 \mu_2^{22} \mu_4 t_{00}^{4} t_{11}^{2} 
+ 1458 \mu_2^{22} \mu_4 t_{00}^{2} t_{11}^{4} 
- 2187 \mu_2^{20} \mu_4^{2} t_{00}^{4} t_{01}^{2} + \cdots
\end{align*}
The full expression is given in the appendix.}
Similar to the previous example, slices of the teacher-metric data discriminant can be visualized as varying focal loci (\Cref{fig:discriminant-examples}, right).
\end{example}

\section{Networks with Quadratic Activations}
\label{sec:quadratic}

{\textbf{Section overview.} We specialize our analysis to networks with quadratic activations, where the function space corresponds to symmetric matrices. This setting allows for sharper results on critical points and loss landscapes. We describe how parameter and function space critical points relate, characterize teacher-student losses as low-rank matrix approximation problems under distribution-induced norms, and present variations of Eckart-Young theorem for Frobenius and Gaussian norms, proved using geometric insights from the previous section. Finally, we show that non-Gaussian norms can lead to exponentially many critical points, highlighting the landscape's strong dependence on the data distribution.}

\vspace{.8em}

For $\mathcal W = (\alpha, W) \in \RR^{r} \times \RR^{r \times n}$, we write  
\begin{equation}
    f_{\mathcal W}(x) = \sum_{i=1}^r \alpha_i \, (w_i\cdot x)^2 = 
    \left\langle \sum_{i=1}^r \alpha_i\, w_i^{\otimes 2}, x^{\otimes 2} \right\rangle_{\!\!F}
    = x^T  \tau_r(\mathcal W) x,
\end{equation}
where $\tau_r(\mathcal W) = W^\top diag(\alpha) W \in \Sym^2(\RR^n)$ is a symmetric matrix. The results from Section~\ref{sec:shallow-polynomial-networks} were largely focused on the landscape in the thick or filling regime, which for $d=2$ both correspond to $r \ge n$. Our goal is now to describe in detail the situation for $r < n$.

\subsection{Critical points of smooth loss functions}
\label{subsec:critical-quadratic}

We begin by specializing and refining some of the general results on the critical points of smooth loss functions from Section~\ref{sec:shallow-polynomial-networks} to the case of $d=2$. In the following, we write $\mathcal S(r;n)$ for the set of $n\times n$ symmetric matrices with rank exactly $r$. The set $\mathcal S(r;n)$ is a smooth manifold of dimension $\frac{1}{2}r(2n-r+1)$ and the embedded tangent space at a point $S \in \mathcal S(r;n)$ is $T_{S}\mathcal S(r;n) = \{SX + X^\top S \colon X \in \RR^{n \times n}\}$~\cite{helmke1995}. 

\begin{proposition}
\label{prop:gradient-function} If $\ell: {\rm Sym}^2(\RR^n) \rightarrow \RR$ is any smooth function, then $S \in \mathcal S(r;n)$ is a critical point for $\ell|_{\mathcal S(r;n)}$ if and only if $\nabla \ell(S) \cdot S = 0$.
\end{proposition}
\begin{proof} 
We show that $T_{S}\mathcal S({r;n})^\perp=\{A  \in S^2(\RR^{{n}}) \colon A \cdot S = 0\}$. Indeed, if $A \cdot S= 0$, then from the cycle property of the trace operator $\langle A, SX + {X^T}S \rangle_F = tr(A S X) + tr(A X^\top S)  = 0$ for any $X \in \RR^{d \times d}$. Conversely, if $\langle A, SX + X^\top S \rangle_F=0$ for all $X$, then $A s_i = 0$ for all {columns} $s_i$ of $S$ which means that $A \cdot S = 0$.
\end{proof}

We next provide analytical expressions for gradients of the loss in parameter space.

\begin{proposition}
\label{prop:gradient_parameters} If $\ell: {\rm Sym}^2(\RR^n) \rightarrow \RR$  is any smooth function and $L = \ell \circ \tau_r$, then for any $\mathcal W = (\alpha, W) \in \RR^r \times \RR^{r \times n}$ we have that
\begin{equation}\label{eq:gradient-parameters-quadratic}
\begin{aligned}
&\nabla_W L(\mathcal W) = 4 diag(\alpha) \cdot W \cdot \nabla \ell(S) \in \RR^{r \times n}, \\[.2cm]
&\nabla_\alpha L (\mathcal W) = 2 Diag(W \cdot \nabla \ell(S) \cdot W^T) \in \RR^r, \\
\end{aligned}
\end{equation}
where $S = \tau_r(\mathcal W)$, $\nabla \ell(S) \in \RR^{n \times n}$ and $Diag$ extracts diagonal elements of a square matrix. In particular, $\mathcal W$ is a critical point for $L$ if and only if
\begin{itemize}
    \item $\nabla \ell(S) w_i = 0$ for all $i \in \{1,\ldots,k\}$ such that $\alpha_i \ne 0$.
    \item $w_i^T \nabla \ell(S) w_i = 0$ for all $i \in \{1,\ldots,k\}$ such that $\alpha_i = 0$.
\end{itemize}
\end{proposition}
\begin{proof}
The expressions in~\eqref{eq:gradient-parameters-quadratic} follow from a direct calculation. The conditions in the second claim are equivalent to $\nabla_W L(\mathcal W) = 0$ and $\nabla_{\alpha} L(\mathcal W) = 0$.
\end{proof}

By combining the two previous results, we can describe the relation between the critical points in parameter space and function space. In general, all critical points in parameter space arise from critical points in the appropriate low-rank manifold.  

\begin{corollary}\label{cor:param-function-quadratic} If $L = \ell \circ \tau_r$ and $\mathcal W = (\alpha, W)$ is a critical point for $L$ such that $S = \tau_r(\mathcal W)$, then $\nabla \ell(S) \cdot S = 0$ holds. In particular, if $r' = rk(S)$, we have that $S$ is a critical point of the restriction of $\ell$ to $\mathcal S(r{'};n)$. More generally, the image of the entire critical set $\tau_{r}(\mathrm{Crit}(L))$ is equal to $\bigcup_{r' \le r} \mathrm{Crit}(\ell|_{\mathcal S(r';n)})$.
\end{corollary}

\begin{proof}
The first statement is a consequence of the first part of Proposition~\ref{prop:gradient_parameters}, noting that if $\mathcal W$ is a critical point for $L$
then $0 = W \cdot \nabla_{W} L(\mathcal W) = 4 \cdot S \cdot \nabla {\ell}(S)$. The second statement then follows from Proposition~\ref{prop:gradient-function}. For the last statement, consider $S$ such that $rk(S) = r' < r$ and $\nabla \ell(S) \cdot S = 0$, so $S \in \mathrm{Crit}(\ell|_{\mathcal S(r';n)})$. A parameterization $S = \tau_r(\mathcal W)$ such that $\alpha_i \ne 0$ for exactly $r'$ indices will be such that then $\nabla L(\mathcal W) = 0$ (again from Proposition~\ref{prop:gradient_parameters}). It {is} enough to note that any $S$ with $rk(S) = r' < r$ admits a representation in parameter space with exactly $r'$ non-zero neurons.
\end{proof}

\subsection{Teacher-student problems}
\label{subsec:teacher-student-quadratic}

A teacher-student problem with quadratic activations is determined by an inner product on $\Sym^2(\RR^n)$. As in Section~\ref{sec:teacher-student}, we consider either 1) the Frobenius inner product $\langle S, T \rangle_F = tr(ST)$, or 2) an inner product $\langle S, T \rangle_{\mathcal D}$ induced by some data distribution $\mathcal D$ in $\RR^n$. In general, we have that
\begin{equation}\label{eq:inner-product-dist-quadratic}
    \langle S, T \rangle_{\mathcal D} = \sum_{i,j,k,l} \mu_{ijkl} S_{ij} T_{kl}, \qquad \mu_{ijkl} = \mathbb E[x_i x_j x_k x_l].
\end{equation}
For special distributions, the inner product has a simpler expression.

\begin{proposition}\label{prop:iid-dist}
If the distribution $\mathcal D$ is such that each coordinate $x_i$ is i.i.d. with $\mathbb E[x_i]=0$, then 
\begin{equation}\label{eq:iid-quadratic}
\langle S, T \rangle_{\mathcal D} = 2 \mu_2^2 \, tr(S T) + \mu_2^2 \, tr(S)tr(T)
+ (\mu_4 - 3 \mu_2^2) \, tr(S \ast T),
\end{equation}
where $\mu_2 = \mathbb E[x_i^2]$ $\mu_4 = \mathbb E[x_i^4]$ and $\ast$ denotes the Hadamard product. In particular, for centered Gaussian or rotationally symmetric distributions, $\langle S, T \rangle_{\mathcal D} \propto 2 tr(ST) + tr(S) tr(T)$.
\end{proposition}
\begin{proof} The expression~\eqref{eq:iid-quadratic} essentially appears in~\cite[Theorem 3.1(b)]{gamarnik2020}, but as a norm rather than as an inner product. Using~\eqref{eq:inner-product-dist-quadratic} and the assumptions on $\mathcal D$, we have that
\[
\begin{aligned}
\langle S, T \rangle_{\mathcal D} &= \sum_{i=1}^d \mu_{iiii} S_{ii} T_{ii} + \sum_{\substack{i,j=1\\i \ne j}}^d \mu_{iijj} S_{ii}T_{jj} + 2\sum_{\substack{i,j=1\\i \ne j}}^d \mu_{ijij} S_{ij}T_{ij} \\
&=\mu_4 \,tr(S \ast T) + \mu_2^2\, \left(tr(S)tr(T) - tr(S \ast T)\right) +  2 \mu_2^2 \, \left(tr(S T) - tr(S \ast T)\right) \\[.5cm]
&=(\mu_4 - 3 \mu_2^2) \, tr(S \ast T) + \mu_2^2 \, tr(S)tr(T) +  2 \mu_2^2 \,tr(S T). \\
\end{aligned}
\]
For standard Gaussians we have $\mu_2=1,\mu_4=3$. Morever, as observed in Section~\ref{sec:teacher-student}, all rotationally symmetric distributions induce the same inner product up to a positive scale factor.
\end{proof}

From Corollary~\ref{cor:param-function-quadratic}, it is sufficient to describe critical points of the loss restricted to the low-rank manifold $\mathcal S(r;n)$. Thus, we consider teacher-student {problems} in function space:
\[
\min_{S \in \mathcal S({r;n})} h_T(S), \qquad h_T(S) = \|T-S\|^2,
\]
where $T \in \Sym^2(\RR^n)$ is fixed and $\|\cdot \|$ is a norm induced by one of the inner products above. 

For the Frobenius inner product, the optimization landscape is described by the following version of Eckart-Young's Theorem, shown in~\cite{helmke1995}. We present here a much shorter geometric proof using the general description of critical points for distance functions given in Section~\ref{sec:teacher-student} (the proof in~\cite{helmke1995} consists in more than 3 pages of computations). We make use of the following intuitive description of the focal locus.

\begin{proposition}\label{prop:focal-locus-quadratic}
Let $X = \mathcal S(r,n)$ and let $\Sigma$ correspond to the Frobenius or Gaussian norm. Given $S \in X$, consider $A \in \Sym^2(\RR^n)$  such that $\{T_\gamma = S + \gamma A \colon \gamma \ge 0\}$ is a $\Sigma$-orthogonal ray to $X$. Assume further that there exists a matrix along this ray with distinct eigenvalues. Then $T_\gamma$ is a focal point for $(X,S)$ if and only if it has a repeated eigenvalue. More precisely, the multiplicity of $T_\gamma = \sum_{i=1}^n \sigma_i (u_i \otimes u_i)$ as a focal point is given by $\# \{(j,k) \colon j < k, \sigma_j = \sigma_k\}$.
\end{proposition}

We refer to the appendix for a proof of this result. In the following, for any $\mathcal I \subset \{1,\ldots,n\}$ and diagonal matrix $\Sigma$, we write $\Sigma_{\mathcal I}$ for the diagonal matrix such that $\Sigma_{\mathcal I,ii} =\Sigma_{ii}$ if $i \in \mathcal I$ and $\Sigma_{\mathcal I,ii} = 0$ otherwise. We also write $\mathcal I^c$ for $\{1,\ldots,n\} \setminus \mathcal I$.

\begin{theorem}[Eckart-Young Theorem for symmetric matrices]\label{thm:eckart-young-fro} Assume that $T \in \Sym^2(\RR^n)$ has full rank $n$ and distinct eigenvalues and let $T = U \Sigma U^T$ be an eigendecomposition of $T$.
Using the Frobenius inner product, we have that:
\begin{itemize}
    \item The function $h_T|_{\mathcal S(r;n)}: \mathcal S(r;n) \rightarrow \RR$ has exactly $\binom{n}{r}$ critical points, given by $S_{\mathcal I} = U \Sigma_{\mathcal I} U^T$ for any $\mathcal I \subset \{1,\ldots,n\}$ such that $\# \mathcal I = r$.
    \item  The index of the critical point $S_{\mathcal I} = U \Sigma_{\mathcal I} U^T$ is given by
    \begin{equation}\label{eq:index-frobenius}
    {\rm index}(S_{\mathcal I}) = \#\left\{(i,j) \in \mathcal I \times \mathcal I^c \colon 0 \le {\sigma_i}/{\sigma_j} \le 1\right\}.
    \end{equation}
\end{itemize}
\end{theorem}
\begin{proof}  We have that $S \in {\mathcal S(r;n)}$ is a critical point of $h_T|_{\mathcal S(r;n)}$ if and only if $\langle T-S, XS + SX\rangle = 0$ for all $X \in \Sym^2(\RR^n)$, which in turn is equivalent to $S \cdot (T-S)=0$ by Proposition~\ref{prop:gradient-function}. Now let $S = \sum_{i=1}^{r} \sigma_i'(u_i'\otimes u_i')$ and $(T-S) = \sum_{j=1}^{s} \sigma_j''( u_j'' \otimes u_j'')$ be eigendecompositions of $S$ and $T-S$ with $\sigma_i'\ne 0$ and $\sigma_j''\ne 0$. The row/column spaces of $S$ and $T-S$ are spanned by the vectors $u_i'$ and $u_j''$, respectively. From the condition $S(T-S)=0$, we deduce that $S$ is a critical point if and only if $u_i'\cdot u_j''=0$ for all $i,j$ in the two decompositions, which means that
\[
T = S + (T - S) = \sum_{i=1}^{r} \sigma_i' (u_i' \otimes u_i') + \sum_{i=1}^{s} \sigma_i'' (u_i'' \otimes u_i'')
\] 
is an eigendecomposition of $T$. This shows that critical points are of the form $S=U \Sigma_{\mathcal I} U^T$ where $U \Sigma U^T$ is an eigendecomposition of $T$ and $\mathcal I \subset \{1,\ldots,n\}$ is such that $\# \mathcal I = r$.

To prove the second point, we use Theorem~\ref{thm:focal_points} and Proposition~\ref{prop:focal-locus-quadratic}, which imply that the index of a critical point $S_{\mathcal I}$ for $h_T|_{\mathcal S(r;n)}$ is the number of matrices on the line segment connecting $S_{\mathcal I}$ and $T$ which have pairs of repeated eigenvalues. It follows from the first point that for any $\alpha \in [0,1]$, the eigenvalues of $M_\alpha = \alpha T + (1-\alpha) S$ are the diagonal elements of $\Sigma_{\mathcal{I}} + \alpha \Sigma_{\mathcal{I}^c}$. 
Two eigenvalues of $M_\alpha$ are the same if and only there exist $i \in \mathcal{I}$ and $j \in \mathcal{I}^c$ such that $\alpha = {\sigma_i}/{\sigma_j}$. Thus, the number of intersections with the discriminant locus is equal to the number of pairs $(\sigma_i, \sigma_j) \in \mathcal I \times \mathcal I^c$ such that ${\sigma_i}/{\sigma_j} \in [0,1]$.
\end{proof}
We now provide a similar precise description of the landscape in the case of matrix norms induced by Gaussian data distributions (or rotationally invariant distributions), \ie, using the distance function $h_T(S) = 2\langle S-T,S-T \rangle_F + tr(S-T)tr(S-T)$.

\begin{theorem}[Eckart-Young Theorem for symmetric matrices and Gaussian norm]\label{thm:eckart-young-Gauss} In the same setting as Theorem~\ref{thm:eckart-young-fro}, but using the Gaussian norm, we have that:
\begin{itemize}
    \item The function $h_T|_{\mathcal S(r;n)}: \mathcal S(r;n) \rightarrow \RR$ has exactly $\binom{n}{r}$ critical points, given by $S_{\mathcal I} = U \Lambda_{\mathcal I} U^T$ for any $\mathcal I \subset \{1,\ldots,n\}$ such that $\# \mathcal I = r$, with
\begin{equation}\label{eq:lambda-eig}
    \Lambda_{\mathcal I,ii} = \begin{cases}
    \Sigma_{ii} + c_{\mathcal I} & \mbox{if } i \in \mathcal I\\ 
    0 & \mbox{if } i \not \in \mathcal I,
    \end{cases}
    \end{equation}
    where $c_{\mathcal I} = \frac{\sum_{i\in \mathcal I^c} \sigma_i}{r+2}$.
    \item The index of the critical point $S_{\mathcal I} = U \Lambda_{\mathcal I} U^T$ is given by
    \begin{equation}\label{eq:index-gaussian}
    {\rm index}(S_{\mathcal I}) = \#\left\{(i,j) \in \mathcal I \times \mathcal I^c \colon 0 \le \frac{\sigma_i + c_{\mathcal I}}{\sigma_j + c_{\mathcal I}} \le 1\right\}.
    \end{equation}
\end{itemize}
\end{theorem}
\begin{proof} The proof is similar to that of Theorem~\ref{thm:eckart-young-fro}. First, we note that $S \in {\mathcal S(r;n)}$ is a critical point of $h_T|_{\mathcal S(r;n)}$ if and only if $S \cdot \nabla h_T(S) = 0$ where
\[
\nabla h_T(S) = 2(S-T) + tr(S-T) Id.
\]
Assume that $S = \sum_{i=1}^{r} \sigma_i' (u_i' \otimes  u_i')$ and $(S-T) = \sum_{i=1}^{n} \sigma_i'' (u_i'' \otimes u_i'')$. Setting $\tau = tr(S-T)$, we have $\nabla h_T(S) = \sum_{i=1}^n (2\sigma_i'' + \tau) (u_i'' \otimes u_i'')$. From the condition $S \cdot \nabla h_T = 0$, we deduce that $S$ and $\nabla h_T$ have orthogonal eigenvectors, \ie, $u_i' \cdot u_j'' = 0$ holds for all $i$ and all $j$ such that $2 \sigma''_j + \tau \ne 0$. This implies that a decomposition of $(S-T)$ is such that, up to permutation of indices, $u_i'' = u_i'$ and $\sigma_i'' = -\tau/2$ holds for $i=1,\ldots,r$, while $\sum_{i=r+1}^n \sigma''_i = \tau(r+2)/2$ (so that $tr(S-T) = \tau$). We thus have that an eigendecomposition of $T$ is given by
\[
T = S + (T-S) = \sum_{i=1}^{r} (\sigma_i' - \tau/2) (u_i' \otimes u_i') + \sum_{i={r+1}}^n \sigma_i'' (u_i'' \otimes u_i'').
\]
In particular, $S$ is obtained from the eigendecomposition of $T$ by choosing (any) $r$ eigenvectors $u_1,\ldots, u_{r}$ and by replacing each eigenvalue $\sigma_i$ with $\sigma_i' = \sigma_i + \tau/2$ where $\tau(r+2)/2 = \sum_{i=r+1}^n \sigma_i''$ is the sum of the remaining eigenvalues of $T$. This yields the characterization in~\eqref{eq:lambda-eig}.

To prove the second point, we use again Theorem~\ref{thm:focal_points} and Proposition~\ref{prop:focal-locus-quadratic}. The eigenvalues of $M_\alpha = \alpha T + (1-\alpha) S$ are the diagonal elements of $\Sigma_{\mathcal{I}} + \alpha \Sigma_{\mathcal{I}^c} + (1-\alpha) c_{\mathcal I} I_{\mathcal{I}}$ where $c_{\mathcal I} = \frac{tr(\Sigma_{\mathcal{I}^c})}{r+2}$ and $I_{\mathcal I}$ is the diagonal matrix with $I_{\mathcal I,ii}=1$ if $i \in \mathcal I$ and $I_{\mathcal I,ii}=0$ otherwise. Two eigenvalues of $M_{\alpha}$ are the same if and {if} only there exist $i \in \mathcal{I}$ and $j \in \mathcal{I}^c$ such that $\sigma_i + (1-\alpha)c_{\mathcal I} = \alpha \sigma_j$, or equivalently $\alpha = \frac{\alpha_i + c_{\mathcal I}}{\alpha_j + c_{\mathcal I}}$. The number of intersections with the discriminant locus is thus equal to the number of pairs $(\sigma_i, \sigma_j) \in \mathcal I \times \mathcal I^c$ such that ${(\sigma_i + c_{\mathcal I})}/{(\sigma_j + c_{\mathcal I})} \in [0,1]$.
\end{proof}

{A statement analogous to Theorem~\ref{thm:eckart-young-Gauss} can be derived in the same way for any matrix inner product of the form $\langle S,T\rangle = \alpha {tr}(ST) + \beta {tr}(S)\mathrm{tr}(T)$, because Proposition~\ref{prop:focal-locus-quadratic} also applies in these cases. In particular, this yields a complete description of the landscape for all distribution-based inner products associated with distributions satisfying $\mu_4 = 3\mu_2^2$.} However, for more general inner products, there may be exponentially more critical points.

\begin{theorem}\label{thm:iid-case}
There exist non-Gaussian distributions $\mathcal{D}$ on $\mathbb{R}^n$ with i.i.d. coordinates and positive Lebesgue-measure subsets $ \mathcal{T} \subset \mathbb{R}^n$ such that the optimization problem
\begin{equation}\label{eq:optim-problem}
 \min_{\substack{S \in \operatorname{Sym}^2(\mathbb{R}^n
)\\ \operatorname{rank}(S)=1}}\| S - T \|_{\mathcal{D}}^2
\end{equation}
has $(3^n -1)/2$ critical points whenever $T = \operatorname{diag}(t)$ for $t \in \mathcal{T}$.  
\end{theorem}

\Cref{thm:iid-case} is proven in the appendix.  This stands in  contrast to the Gaussian and Frobenius cases for low-rank approximation, where the corresponding optimization problem has only $n$ critical points generically. The number $(3^n -1)/2$ in \Cref{thm:iid-case} coincides with the ``generic ED degree'' reported in~\cite[Example 5.6]{draisma2013}.
{A numerical experiment illustrating these results is presented in the appendix.}

\section{Conclusions} 

We investigated various aspects of shallow networks with polynomial activations. We described the function space of these models and explored the relationship between network width and the loss landscape. Additionally, we examined teacher-student problems, relating them to low-rank tensor approximations and showing how certain discriminants encode the relation between data and optimization. We also focused on networks with quadratic activations, presenting variations of the Eckart-Young Theorem for teacher-student problems under different norms.

Several directions could be pursued in future research. First, we hope to further explore the teacher-metric discriminant and its applications in machine learning. It would also be desirable to close the gap between negative and positive results on landscape and width, 
beyond the cases of $d=2$ or $n=2$ which were mainly considered in this work. Teacher-student problems could be investigated in detail in more diverse settings, such as with higher activation degrees or distributions beyond Gaussians (e.g., mixtures of Gaussians). More broadly, we have suggested the possible applicability of our analysis to networks with non-polynomial activations, but such connections remain to be developed.

\paragraph{Acknowledgments.} J.K. and E.P. were supported partially by NSF DMS 2309782, NSF CISE-IIS 2312746  and DE SC0025312, as well as start-up grants and a JTO fellowship from the University of Texas at Austin.
YA was partially supported by the Israel Science Foundation (grant No. 724/22).
\vspace{.5cm}

\small{
\noindent\textit{Yossi Arjevani:} The Hebrew University (\texttt{yossi.arjevani@gmail.com}) \\[2pt]
\textit{Joan Bruna:} New York University (\texttt{bruna@cims.nyu.edu}) \\[2pt]
\textit{Joe Kileel:} University of Texas at Austin (\texttt{jkileel@math.utexas.edu}) \\[2pt]
\textit{Elzbieta Polak:} University of Texas at Austin (\texttt{epolak@utexas.edu}) \\[2pt]
\textit{Matthew Trager:} AWS AI Labs, work done outside of Amazon (\texttt{matthewtrager@gmail.com})
}

\bibliographystyle{plain}
\bibliography{references}
\newpage

\appendix
{
\section{Numerical Experiment}
We present a numerical experiment to illustrate the results of Section~\ref{sec:quadratic}. We fix a teacher quadratic model \(\mathcal{V} = (\beta, V) \in \mathbb{R}^5 \times \mathbb{R}^{5 \times 5}\) defined by  
\[
\beta = \begin{bmatrix} -4 & -2 & 1 & 3 & 5 \end{bmatrix}, \qquad V = I - \tfrac{2}{5} J,
\]
where \(I \in \mathbb{R}^{5 \times 5}\) is the identity matrix and \(J \in \mathbb{R}^{5 \times 5}\) is the all-ones matrix. Since \(V\) is orthogonal, the eigenvalues of the teacher matrix \(T = V^\top \operatorname{diag}(\beta) V \in \operatorname{Sym}^2(\mathbb{R}^5)\) are precisely \(\beta\). 
We generate a dataset \(\{(x_i, y_i)\}_{i=1}^N \subset \mathbb{R}^5 \times \mathbb{R}\) of size \(N = 50,000\), where inputs \(x_i \in \mathbb{R}^5\) are sampled from a standard Gaussian distribution, and outputs are computed as \(y_i = f_{\mathcal{V}}(x_i) = x_i^\top T x_i\). This dataset is used to train a student model \(f_{\mathcal{W}}\) of width \(r = 3\) by minimizing the mean squared error loss $L(\mathcal{W}) = \frac{1}{N} \sum_{i=1}^N |y_i - f_{\mathcal{W}}(x_i)|^2$.
We use stochastic gradient descent (SGD) with a batch size of 256, a learning rate of \(10^{-4}\), and train for 500 epochs, with Pytorch's default initialization of linear layers. The experiment is repeated 1000 times, fixing the teacher model \(\mathcal{V}\) across all trials.\\
From Corollary~\ref{cor:param-function-quadratic}, the critical points in parameter space correspond to the union of the critical points in function space of the functional loss restricted to the smooth loci of the low-rank manifolds \(\mathcal{S}(r, 5)\) for \(r \leq 3\), i.e. $\tau_r(\operatorname{Crit}(L)) = \bigcup_{r \leq 3} \operatorname{Crit}(h_T|_{\mathcal{S}(r, 5)})$.
Here, the functional loss is given by the empirical risk squared norm \(h_T(S) = \|T - S\|^2_{ERM}\)~\eqref{eq:product_ERM}, which in the large-sample limit converges to the Gaussian norm \(h_T(S) = 2\langle S - T, S - T \rangle_F + \operatorname{tr}(S - T)^2\). Within each manifold \(\mathcal{S}(r, 5)\), all critical points and their indices are then completely described by Theorem~\ref{thm:eckart-young-Gauss}. In our case, this results in a total of \(\binom{5}{3} + \binom{5}{2} + \binom{5}{1} = 25\) non-zero critical points. Computing the indices of these critical points using~\eqref{eq:index-gaussian}, we find that there are \(2 + 2 + 2 = 6\) local minima across the manifolds \(\mathcal{S}(3, 5)\), \(\mathcal{S}(2, 5)\), and \(\mathcal{S}(1, 5)\), respectively. Note, however, that these low-rank minima are not necessarily minima in parameter space, though any minimum in parameter space must correspond to one of them. For example, in the notation of Theorem~\ref{thm:eckart-young-Gauss}, the global minimum corresponds to the index set $\mathcal I = \{1,4,5\}$, which defines the student matrix $S_{\mathcal I} = V^{\top} {\rm diag}(\gamma_{\mathcal I}) V$, where
\[
\gamma_{\mathcal I} = \begin{bmatrix} -4.2 & 0 &  0 & 2.8  & 4.8 \end{bmatrix}.
\]
The values of \(\gamma_{\mathcal{I}}\) are computed by adding \(c_{\mathcal{I}} = \frac{-2 + 1}{5} = -0.2\) to the selected eigenvalues of the teacher matrix.\\
Figure~\ref{fig:eigenvalue_histogram_data} shows the eigenvalue histograms of the student models at the end of training, which align closely with the eigenvalues of the theoretically predicted minima. The average Frobenius distance between trained students and their nearest theoretical minima was approximately \(8 \times 10^{-2}\). We also performed direct optimization of the Gaussian norm (without sampling data) with a learning rate of \(10^{-3}\) and 10,000 iterations. The resulting eigenvalue distributions (Figure~\ref{fig:eigenvalue_histogram_norm}) again match the predicted minima and agree closely with the sample-based results.  Figure~\ref{fig:nearest_minima_counts} shows the frequency with which each theoretical minimum is closest to the trained student models. Only 4 of the 6 minima were actually reached in our trials. Finally, we repeated the experiment, replacing the standard Gaussian distribution used to generate inputs with a coordinate-wise uniform distribution in ([-1, 1]). Using the exact same training parameters, the final student models were significantly more dispersed, in line with Theorem~\ref{thm:iid-case}: with a tolerance of \(0.05\) Frobenius norm, there were 42 distinct students at the end of training, compared to only 5 in the Gaussian setting.
}

\begin{figure}[h!]
    \centering
    \begin{subfigure}[b]{0.45\linewidth}
        \includegraphics[width=\linewidth]{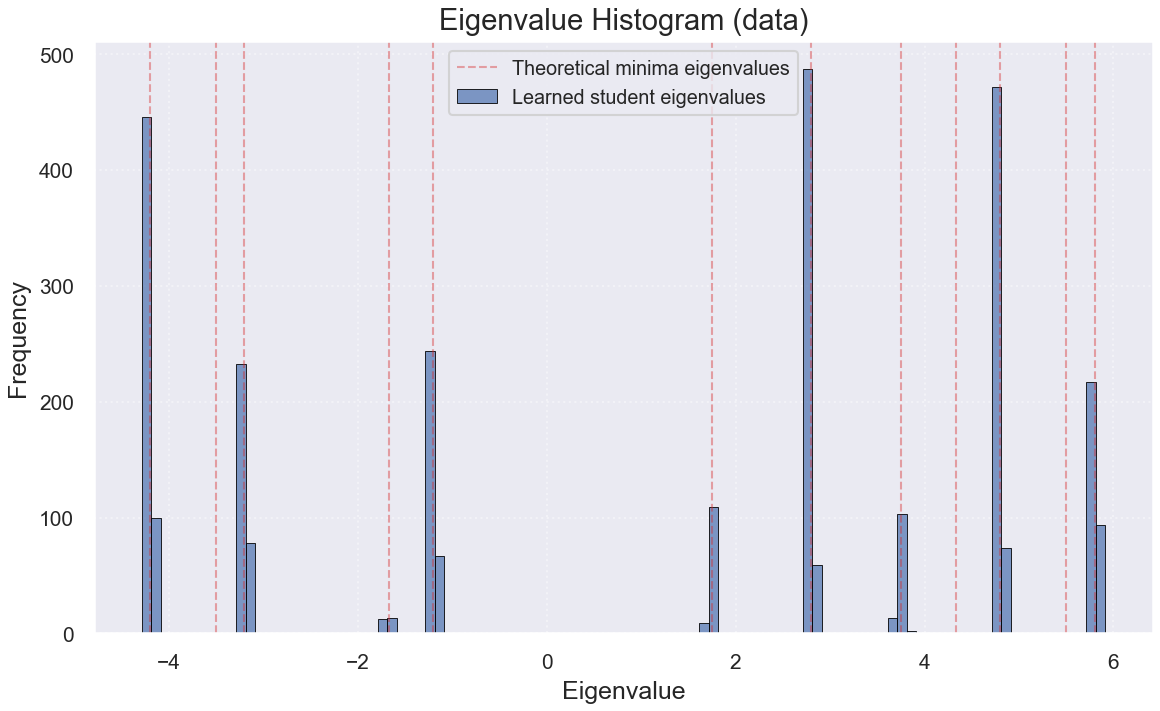}
        \caption{Eigenvalue histogram for final students trained with sample-based mean squared error loss.}
        \label{fig:eigenvalue_histogram_data}
    \end{subfigure}
    \hfill
    \begin{subfigure}[b]{0.45\linewidth}
        \includegraphics[width=\linewidth]{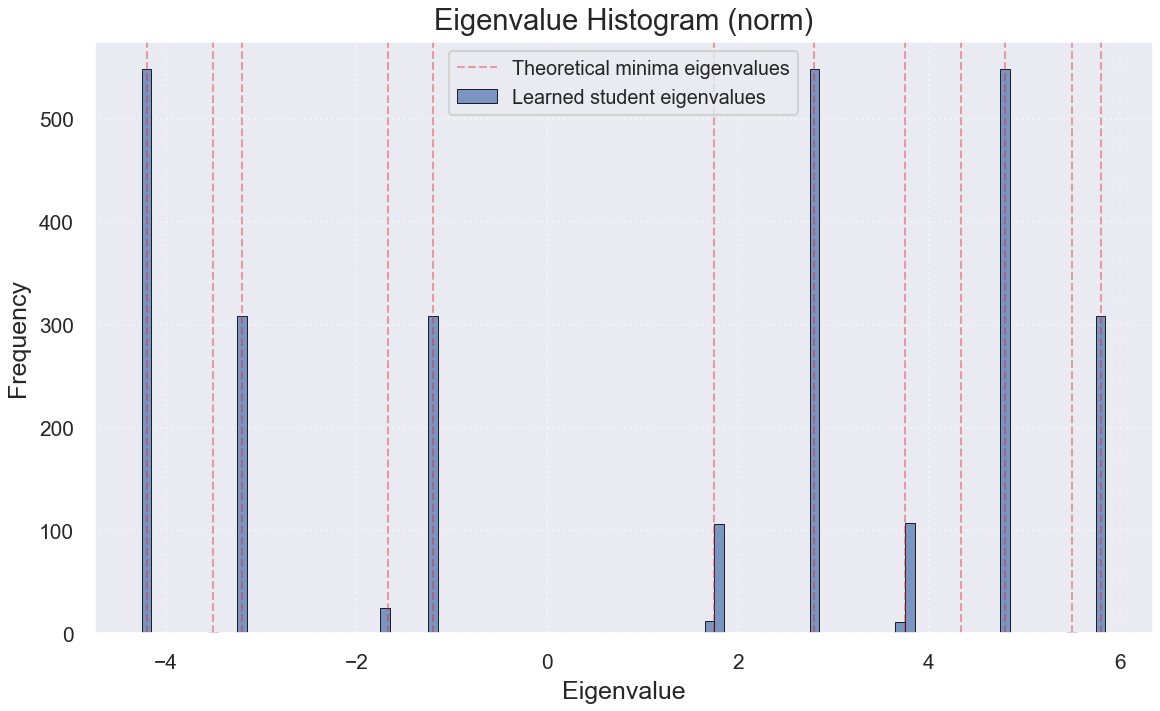}
        \caption{Eigenvalue histogram for final students trained with Gaussian square norm minimization.}
        \label{fig:eigenvalue_histogram_norm}
    \end{subfigure}
    \vspace{0.5cm}
    \begin{subfigure}[b]{0.9\linewidth}
        \centering
        \includegraphics[width=0.6\linewidth]{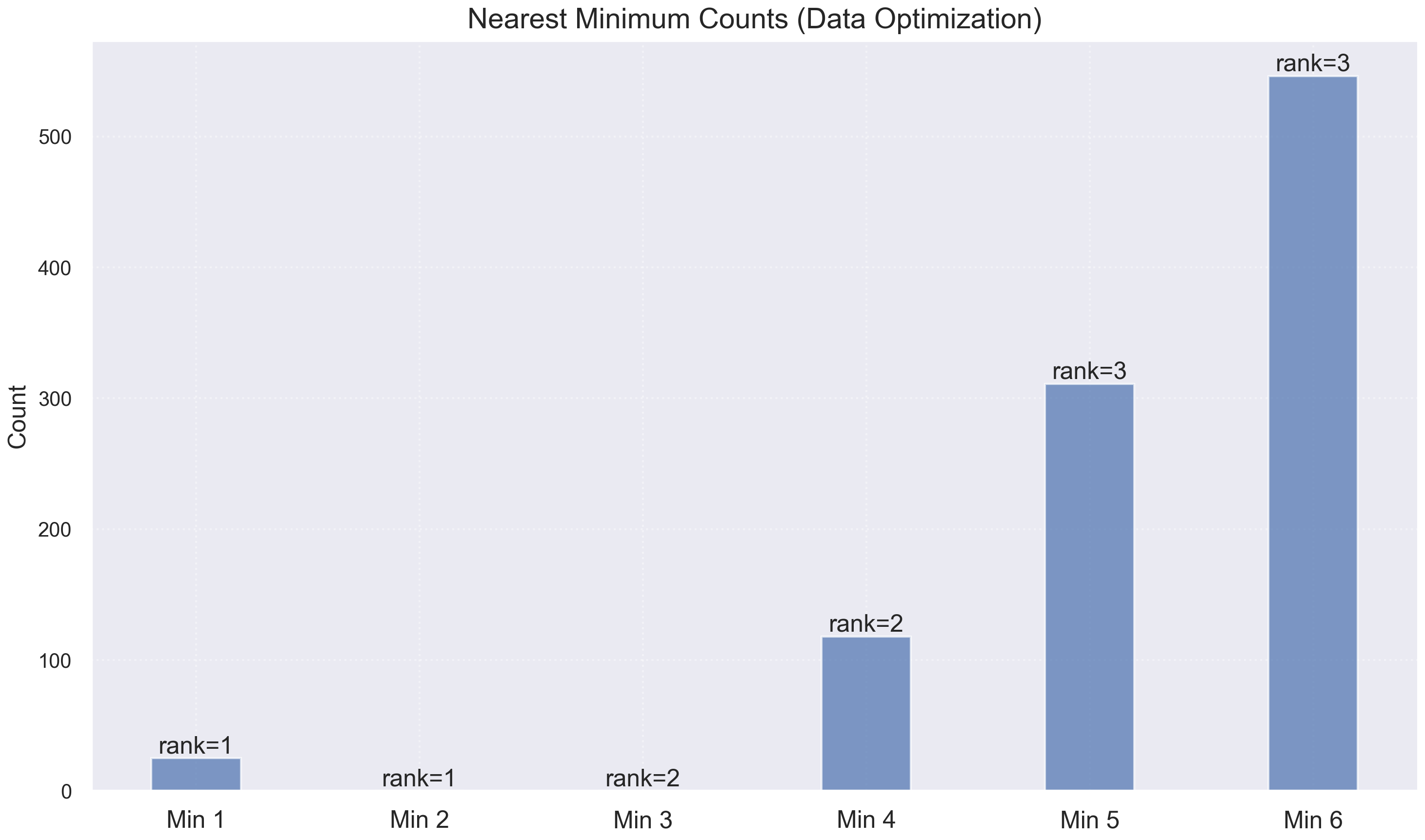}
        \caption{Counts of nearest minima reached by student models, where `Min 6' corresponds to the global minimum.}
        \label{fig:nearest_minima_counts}
    \end{subfigure}
    \caption{{(a)–(b) Eigenvalue histograms of student models after training, overlaid with the predicted minima. Both sample-based and norm-based optimization yield nearly identical spectra that closely match the eigenvalues of the predicted minima. (c) Frequencies of minima reached (sample-based setting; results are nearly identical for the norm-based setting)}}
    \label{fig:results}
\end{figure}

\section{Proofs and Additional Results}

For quadratic networks, we prove 
the following topological description of the fiber.

\begin{theorem}\label{thm:fiber_quadratic} Assume $S \in 
\Sym^2(\RR^n)$ has signature $(s_+, s_-, s_0)$. The number of connected 
components of the fiber
\[
\tau_r^{-1}(S) = \{(\alpha,W) \colon W^T \diag(\alpha) W = S\}\subset \RR^r \times \RR^{r \times n},
\]
is as follows:
\begin{enumerate}
\item if $\rk(S) > r$, then $\tau_r^{-1}(S)$ is empty;
\item if $\rk(S) = r$ and $s_+=0$ or $s_- = 0$, then $\tau_r^{-1}(S)$ has $2$ connected components;
\item if $\rk(S) = r$, $s_+>0$, {and}  $s_- > 0$, then $\tau_r^{-1}(S)$ has $4 \cdot \binom{r}{s_+} = 4 \cdot \binom{r}{s_-}$ connected components;
\item if $\rk(S) < r$, then $\tau_r^{-1}(S)$ is connected.
\end{enumerate}
\end{theorem}

\begin{lemma}\label{lemma:signature_fiber} Let $S \in \Sym^2(\RR^n)$ be as in Theorem~\ref{thm:fiber_quadratic} and consider $(\alpha, W) \in \tau_r^{-1}(S)$. If $\alpha_+, \alpha_-$ are the number of (strictly) positive and negative elements in $\alpha$ respectively, then $\alpha_+ \ge s_+$, $\alpha_- \ge s_-$.
\end{lemma}
\begin{proof} It is enough to observe that $W^T \diag(\alpha) W$ can have at most $\alpha_+$ strictly positive eigenvalues. Indeed, $U = \{v \in \RR^n\,\, \colon \,\, (W v)_i = 0, \,\, \forall i \,\, s.t. \,\, \alpha_i > 0\}$ is a vector space of dimension at least $n-\alpha_+$ where the restriction of the quadratic form $v \mapsto v^T W^T \diag(\alpha) W v$ is negative semi-definite. The same argument shows that $W^T \diag(\alpha) W$ can have at most $\alpha_-$ strictly negative eigenvalues.
\end{proof}

\begin{proof}[Proof of Theorem~\ref{thm:fiber_quadratic}] If $\rk(S) > r$ then $\tau_r^{-1}(S)$ is empty since any matrix of the form $W^T \diag(\alpha) W$ has rank at most $r$. Assume now $\rk(S) = r$. If $(\alpha,W) \in \tau_r^{-1}(S)$, then all elements of $\alpha \in \RR^r$ are non-zero, and by Lemma~\ref{lemma:signature_fiber} exacly $s_+$ elements are positive and $s_-$ elements are negative. Equivalently, if we decompose the parameter space based on the sign of the elements of $\alpha$:
\begin{equation}\label{eq:decomposition_fiber}
\RR^r \times \RR^{r \times n} = \bigsqcup_{\sigma \in \{0,1,-1\}^r} \mathcal U_{\sigma}, \qquad \mathcal U_{\sigma} = \{(\alpha,W) \colon sgn(\alpha_i) = \sigma_i \},
\end{equation}
then $\tau_r^{-1}(S) \cap \mathcal U_{\sigma}$ is not empty if and only if $|\{i \colon \sigma_i = 1\}| = s_+$, $|\{i \colon \sigma_i = -1\}| = s_-$, and $|\{i \colon \sigma_i = 0\}| = 0$. This yields $\binom{r}{s_+} = \binom{r}{s_-}$ vectors $\sigma$ such that $\tau_r^{-1}(S) \cap \mathcal U_{\sigma}$ is not empty. Because $\mathcal U_{\sigma}$ with $\sigma_i \ne 0$ are disconnected, the connected components of $\tau_r^{-1}(S)$ are the union of the connected components of each $\tau_r^{-1}(S) \cap \mathcal U_{\sigma}$. We can thus describe the connected components of the fiber $\tau_r^{-1}(S)$ with the constraint that the signs of $\alpha_i$ are fixed. It is easy to see that we can in fact further assume that $\alpha_i = \pm 1$, since there is always a continuous path within the fiber $\tau_r^{-1}(S)$ that ``normalizes'' each $\alpha_i$:
\[
\alpha_i(t) = (1-t) \alpha_i  + t \frac{\alpha_i}{|\alpha_i|}, \quad w_i(t) =  \frac{w_i}{\sqrt{|\alpha_i(t)|}}, \quad W(t)^T \diag(\alpha(t)) W(t) = S.
\]
Writing $\mathcal V_{\sigma} = \{(\alpha,W) \colon \alpha_i = \sigma_i\}$ (so $\mathcal V_{\sigma} \subset \mathcal U_{\sigma}$), we focus on the connected components of $\tau_r^{-1}(S) \cap \mathcal V_{\sigma}$(formally, $\tau_r^{-1}(S) \cap \mathcal V_{\sigma}$ is a deformation retract of $\tau_r^{-1}(S) \cap \mathcal U_{\sigma}$, so it has the same number of connected components). By symmetry, we may assume that the first $s_+$ elments of $\alpha$ are $1$ and the last $s_-$ elements are $-1$. We consider the set
\begin{equation}\label{eq:indefinite_group}
O(s_+,s_-) = \{M \colon M^T \diag(\sigma) M = \diag(\sigma) \} \subset \RR^{n \times n}.
\end{equation}
If $s_+=0$ or $s_-=0$, this is the orthogonal group, otherwise it is the
``indefinite'' orthgonal group. The group $O(s_+,s_-)$ acts on $\tau_r^{-1}(S)
\cap \mathcal V_{\sigma}$ according to $(M,(\sigma,W)) \mapsto (\sigma, MW)$
since $W^T M^T \diag(\sigma) M W = W^T \diag(\sigma) W$. We claim that this
action is \emph{transitive} (any point in $\tau_r^{-1}(S) \cap \mathcal
V_{\sigma}$ can be mapped to any other) and \emph{free} (only the identity
fixes any point). Indeed, if $W$ and $W'$ are matrices in $\RR^{r \times n}$
such that $W^T \diag(\sigma) W = W'^T \diag(\sigma) W' = S$, then $Row(W) =
Row(W') = Row(S)$, and there exists $M \in GL(r, \RR)$ such that $W' = M W$.
Such $M$ belongs to $O(s_+, s_-)$ since $W^T M^T \diag(\sigma) M W = W^T
\diag(\sigma) W$ implies $M^T \diag(\sigma) M = Id_{r}$ (by multiplying by the
pseudoiverses of $W^T$ and $W$). This proves that the action is transitive.
The action is also free since if $(\sigma, W) \in \tau_r^{-1}(S)$ then $rk(W)
= r$ and $M W = W$ implies $M = Id$. The fact that the action is transitive
and free means that the number of connected components of $\tau_r^{-1}(S) \cap
\mathcal V_{\sigma}$ is the same as the number of connected components of
$O(s_+,s_-)$. It is known that $O(s_+,s_-)$ has two connected components if
$s_+ = 0$ or $s_- = 0$, and four connected components
otherwise~\cite{knapp2013lie}. In the former case, the two connected
components correspond to the sign of the determinant, while in the latter case
the components arise from the combinations of signs of the determinants of the
two diagonal blocks of sizes $s_+ \times s_+$ and $s_- \times s_-$. This
proves cases $2$ and $3$.

We now assume that $rk(S) < r$. We say that $S = W^T \diag(\alpha) W$ is a
``minimal factorization'' of $S$ if $\alpha$ has exactly $\rk(S) = s_+ + s_-$
non-zero elements that are all $1$ or $-1$ and if $w_i=0$ for all $i$ such
that $\alpha_i=0$. In this setting, $\alpha$ must have exactly $s_+$ elements
that are $1$ and $s_-$ elements that are $-1$
(Lemma~\ref{lemma:signature_fiber}). 
We will first prove that for every element $(\alpha,W) \in \tau_r^{-1}(S)$, we
can find a path inside $\tau_r^{-1}(S)$ from $(\alpha,W)$ to a minimal
factorization. Then we will show that any two minimal factorizations of $S$
are path connected in $\tau_r^{-1}(S)$.

For the first claim, we observe as before that it is always possible to
rescale $\alpha_i$ so that all of its elements are $1,-1$ or $0$. We thus need
to show that we may transform $\alpha$ so it has exactly $\rk(S)$ non-zero
elements. We prove this inductively arguing that if $\rk(S)<r$ and $W^T
\diag(\alpha) W = S$, we can find a path inside $\tau_r^{-1}(S)$ from
$(\alpha,W)$ to $(\alpha',W')$ so that $\alpha_i' = w_i' = 0$ for some $i$.
The same argument can then be repeated for the remaining $r'=r-1$ vectors
$w_i$ and elements $\alpha_i$ and iterated until $r' = \rk(S)$. If
$\alpha_i = 0$, then we can simply scale the corresponding $w_i$ and,
similarly, if $w_i=0$ we can scale the corresponding $\alpha_i$. If $\alpha_i
\ne 0$ and $w_i\ne 0$ for all $i$, then necessarily $\rk(W) < r$, so there
exists $v \in \RR^r$ such that $v^T W = 0$. Assume that $\alpha$ has the first
$\alpha_+$ elements equal to $1$, and the last $\alpha_-$ elements equal to
$-1$ (with $\alpha_++\alpha_-=r$). If $\alpha_+=0$ or $\alpha_-=0$, then we
consider a path $(\alpha(t),W(t))$ where $\alpha(t) = \alpha$, $W(t) = M(t) W$
and $M(t)$ is a path in $SO(r)$ such that $M(0) = Id_r$ and the first row of
$M(1)$ is equal to $v/\|v\|$ where $v^T W = 0$ (this path exists because
$SO(r)$ is connected). By construction the first row $M(1)W$ is zero, and we
the can rescale the corresponding $\alpha_1$ to zero, as desired. If
$\alpha_+>0$ and $\alpha_->0$ then we apply essentially the same argument
using the connected components of the indefinite orthogonal group $O(\alpha_+,\alpha_-) = \{M
\in \RR^{r \times r} \colon M^T \diag(\alpha) M = \diag(\alpha) \}$.  However,
we must first consider the special situation that our chosen $v \in
\RR^{r} \setminus \{0\}$ that satisfies $v^T W = 0$ is such that $v^T
\diag(\alpha) v=0$ (\ie, it is an isotropic vector). If this is the case,
we construct a path as follows. Letting $y = \diag(\alpha) v$ and
assuming that $y_i$ is a non-zero element of $y$ (note $y\ne 0$), we set
$W(t) = W - (t/y_i) y w_i^T$. Observing that $\diag(\alpha) y =
v$ and that $y^T \diag(\alpha) y = v^T \diag(\alpha)^3 v = 0$ (because
$\diag(\alpha)
\diag(\alpha) = Id$), we have that
\[
\begin{aligned} &(W - (t/y_i) y w_i^T)^T \diag(\alpha) (W - (t/y_i) y w_i^T) =
W^T \diag(\alpha) W,
\end{aligned}
\] and the $i$-th row of $W - (1/y_i)y w_i^T$ is $0$, as desired. We finally
assume that $v \in \RR^{k}$ for which $v^T W = 0$ is such that $v^T
\diag(\alpha) v \ne 0$. We can normalize $v$ such that $v^T
\diag(\alpha) v = 1$ or $v^T \diag(\alpha) v = - 1$. If $v^T
\diag(\alpha) v = 1$, then we can find a path $M(t)$ in $O(\alpha_+,\alpha_-)$ such that 
$M(0) = Id_r$ and the first row of $M(1)$ is equal to $v$ or $-v$ (this is
true in light of the description of the connected components of
$O(\alpha_+,\alpha_-)$).\footnote{It is helpful to note that $M \in O(\alpha_+,\alpha_-)$ if
and only if $M^T \in O(\alpha_+,\alpha_-)$. Indeed: $M^T \diag(\alpha) M = \diag(\alpha) \Leftrightarrow M^T \diag(\alpha) M \diag(\alpha) M^T = M^T \Leftrightarrow M \diag(\alpha) M^T = \diag(\alpha)$.}  Similarly, if $v^T
\diag(\alpha) v = -1$, we can find a path $M(t)$ in $O(\alpha_+,\alpha_-)$ such that 
$M(0) = Id_r$ and the $(\alpha_+)+1$-th row of $M(1)$ is equal to $v$ or $-v$.
Either way, we have that $(\alpha,W(t)) = (\alpha, M(t) W)$ is a path in
$\tau_r^{-1}(S)$ such that $W(0) = W$ and and $W(1)$ has a zero row. This
concludes the proof of the existence of a path in $\tau_r^{-1}(S)$ from any
$(\alpha,W)$ to a minimal factorization of $S$.

We now prove that any two minimal factorizations of $S$ are path connected in
$\tau_r^{-1}(S)$ if $\rk(S) < r$. We first observe that any two minimal
factorzations $(\alpha,W)$, $(\alpha',W')$ with $\alpha=\alpha'$ are
connected. Indeed, if $W^+ \in \RR^{\alpha_+\times  n}$ and $W'^+ \in
\RR^{\alpha_+ \times n}$ are submatrices of $W$ and $W'$ corresponding to elements of
$\alpha$ that are $1$ (the same argument applies to $W^- \in
\RR^{\alpha_-\times  n}$ and $W'^- \in
\RR^{\alpha_- \times n}$ corresponding to to elements of
$\alpha$ that are $-1$), then we have that $Row(W^+)=Row(W'^+)$, so there
exist $T \in GL(\alpha_+)$ such that $W'^+ = T W^+$. If $\det(T)>0$, then we
can find a path in $GL(\alpha_+)$ from $Id_{\alpha_+}$ to $T$ that transforms
$W'^+$ to $W^+$ without affecting $S$. If $\det(T)<0$, then we use a zero row
of $W$ to change the sign of one row vector $w$ of $W^+$. More precisely, if
$j$ is such that $\alpha_j = 0, w_j=0$, then we consider a path that changes the
sign of $w_i$ while keeping $\alpha_i(t) w_i(t) w_i(t)^T + \alpha_j(t) w_j(t)
w_j(t)^T$ constant:
\begin{equation}\label{eq:flip_path}
\begin{aligned}
&\alpha_i(t)=\begin{cases} \sqrt{1-3t} \cdot \alpha_i & t \in [0,1/3] \\
0 & t \in [1/3,2/3]\\
(3t-2) \cdot \alpha_i& t \in [2/3,1]\\
\end{cases}, 
&& w_i(t)=\begin{cases} \sqrt{1-3t} \cdot w_i & t \in [0,1/3]\\
-\sqrt{3t-1} \cdot w_i & t \in [1/3,2/3]\\
- w_i \cdot \alpha_i& t \in [2/3,1]\\
\end{cases}\\[.3cm]
&\alpha_j(t)=\begin{cases} \sqrt{3t} \cdot \alpha_i & t \in [0,1/3] \\
\alpha_i & t \in [1/3,2/3]\\
\sqrt{3-3t} \cdot \alpha_i& t \in [2/3,1]\\
\end{cases}, 
&& w_j(t)=\begin{cases} \sqrt{3t} \cdot w_i & t \in [0,1/3]\\
w_i & t \in [1/3,2/3]\\
\sqrt{3-3t} \cdot \alpha_i& t \in [2/3,1].\\
\end{cases}\\
\end{aligned}
\end{equation}
If instead we assume $\alpha \ne \alpha'$, then we observe the existence of a
path in $\tau_r^{-1}(S)$ that changes $\alpha$ into $\alpha'$. Indeed, there
exists a path very similar to~\eqref{eq:flip_path} such that $\alpha_j(0)=0$,
$w_j(0)=0$, $\alpha_j(1)=\alpha_i$, $w_j(1)=w_i$ and $\alpha_i(1)=0$, $w_i(1)=
0$ (it is enough to consider the subpath of~\eqref{eq:flip_path} corresponding
to $t \in [0,1/3]$). In other words, this path exchanges a zero row and a
non-zero row. By composing similar paths, we can also exchange a row where the
corresponding element of $\alpha$ is $1$ and one where the corresponding
element of $\alpha$ is $-1$. Using these exchanges between rows, we can reduce
the number of indices where $\alpha$ and $\alpha'$ differ until
$\alpha=\alpha'$ holds. Finally, from the previous analysis, we deduce that
when $\rk(S) < r$ the fiber $\tau_r^{-1}(S)$ is connected. This concludes the
proof of the theorem.
\end{proof}

We remark that the fiber has many more connected components 
over matrices with mixed signature than over semi-definite 
ones. The number of connected components can be viewed as a 
local measure of complexity of the landscape.

\begin{proposition}\label{prop:stiefel-param} Let $V_r(\RR^n):=\{U \in \RR^{r \times n}: U U^T={\rm Id}_r\}$ (Stiefel manifold) and consider the map
\[
\rho_r: \RR^r \times V_r(\RR^n) \rightarrow \Sym^2(\RR^n), \quad (\alpha, U) \mapsto U^T \diag(\alpha) U.
\]
(Note this is a restriction of our parameterization map $\tau_r$). Then the nullity of the differential $d \rho_r$ at $(\alpha, U)$ is equal to $m:=\#\{(i,j) \colon i < j, \alpha_i = \alpha_j\}$. If $m=0$, then $\rho_r$ is a local parameterization of $\mathcal S(r,n)$.
\end{proposition}

\begin{proof}  
 Writing $u_1,\ldots,u_r$ for the rows of $U$, we have that
\begin{equation}\label{eq:drho}
d \rho_r(\alpha, U) = \sum_{i=1}^r \dot \alpha_i (u_i \otimes u_i) + \sum_{i=1}^r \alpha_i (u_i \otimes \dot u_i + \dot u_i \otimes u_i).
\end{equation}
From the condition $U U^T={\rm Id}_r$, the tangent space of $V_r(\RR^n)$ is characterized by $\dot u_i^\top u_i = 0$ and $\dot u_i^\top u_j + \dot u_j^\top u_i = 0$ for all $i \ne j$ in $\{1,\ldots,r\}$. Assuming without loss of generality that $U=[{\rm I}_r \,\, 0]$, the image of~\eqref{eq:drho} is an $n \times n$ symmetric matrix with $\dot \alpha_i$ in the first $r$ entries of the diagonal and $\alpha_i(u_i \otimes \dot u_i)_{ij} + \alpha_j(\dot u_j \otimes u_j)_{ij} = \alpha_i (\dot u_i)_j + \alpha_j (\dot u_j)_i =  (\alpha_i - \alpha_j) (\dot u_i)_j$ in the $(i,j)$-th entry with $i \ne j$ in $\{1,\ldots,r\}$ {(since $u_i, u_j$ are rows of $U=[{\rm I}_r, 0]$ and $\dot u_i^\top u_j + \dot u_j^\top u_i = (\dot u_i)_j + (\dot u_j)_i = 0$)}. From this description, the claim follows easily.
\end{proof}

\begin{proof}[Proof of~\Cref{prop:focal-locus-quadratic}]
If $\Sigma$ corresponds to the Frobenius norm, then ${N_{X,\Sigma}}=\{(S,A) \colon S \in X, \, S\cdot A = 0\}$. Our assumptions on $S$ and $A$ imply that $S = \sum_{i=1}^r \sigma_i'(u_i' \otimes u_i')$ and $A = \sum_{j=1}^{n-r} \sigma_j'' (u_j'' \otimes u_j'')$ each have distinct nonzero eigenvalues. It follows now from Proposition~\ref{prop:stiefel-param} that $(\sigma', \sigma'', u', u'') \mapsto \left(\sum_{i=1}^r \sigma_i'(u_i' \otimes u_i'), \sum_{j=1}^{n-r} \gamma \sigma_j'' (u_j'' \otimes u_j'')\right)$ is a smooth local parameterization of $N_{X,\Sigma}$ in a neighborhood of $(S, \gamma A)$. Since $u', u''$ are orthogonal, another application of Proposition~\ref{prop:stiefel-param} implies that the nullity of the differential of the endpoint map $(S,\gamma A) \mapsto S + \gamma A$ is given by the number of pairs of repeated eigenvalues of $T_\gamma = S + \gamma A$.

If $\Sigma$ corresponds to the Gaussian norm, then $N_{X,\Sigma}=\{(S,A) \colon S \in X, \, S\cdot (2A + {\rm tr}(A) {\rm Id}) = 0\}$. Similar to the previous case, a smooth local parameterization of $N_{X,\Sigma}$ in a neighborhood of $(S, \gamma A)$ is given by $(\sigma', \sigma'', u', u'') \mapsto \left(\sum_{i=1}^r \sigma_i'(u_i' \otimes u_i'), - \sum_{i=1}^r \gamma c (u_i' \otimes u_i') + \sum_{j=1}^{n-r} \gamma \sigma_j'' (u_j'' \otimes u_j'')\right)$ where $c = \frac{1}{r+2}\sum_{j=1}^{n-r} \sigma_j''$ and $u'$ and $u''$ are orthogonal (see the proof of Theorem~\ref{thm:eckart-young-Gauss} for a derivation). It follows again from Proposition~\ref{prop:stiefel-param} the nullity of the differential of the endpoint map is the number of pairs of repeated eigenvalues of $T_\gamma = S + \gamma A$. Here we reduce ourselves to the previous setting by applying the invertible linear map $(\sigma', \sigma'') \mapsto (\sigma' - \gamma c, \sigma'')$ to obtain the eigenvalues of $T_\gamma$.
\end{proof}

\begin{proof}[Proof of~\Cref{thm:iid-case}]
We analyze the first-order criticality condition for \eqref{eq:optim-problem}.
By \Cref{prop:iid-dist}, the Euclidean gradient for i.i.d. inner products is 
\begin{equation}
\nabla \ell(S)  = 4 \mu_2^2 (S-T) + 2\mu_2^2 \operatorname{tr}(S-T)I + (2\mu_4 - 6 \mu_2^2) \operatorname{Diag}(S-T),
\end{equation}
where $\operatorname{Diag}(\cdot)$  zeros out the off-diagonal part of a matrix.  The criticality equations read
\begin{equation} \label{eq:iid-crit}
0 = 4 \mu_2^2 (S-T)S + 2\mu_2^2 \operatorname{tr}(S-T)S+ (2\mu_4 - 6 \mu_2^2) \operatorname{Diag}(S-T)S
\end{equation}
by \Cref{prop:gradient-function}.  
The real solutions to \eqref{eq:iid-crit} with the correct rank are the critical points.
Write $S = \varepsilon v v^{\top}$ for $v \in \mathbb{R}^n \setminus \{0\}$  and $\varepsilon \in \{-1, 1\}$.  
Substituting this and $T = \operatorname{diag}(t)$ into \eqref{eq:iid-crit}, the condition is
\begin{equation} \label{eq:simplied-rk1}
    \left( 6 \mu_2^2 {\|v\|^2}  - 2 \varepsilon \mu_2^2 \langle t, \mathbf{1} \rangle\right) vv^{\top} 
    + \left( 2\mu_4 - 6 \mu_2^2 \right) \operatorname{diag}(v^{\ast 2}) v v^{\top} 
    + \left(- 2 \mu_4 + 2 \mu_2^2 \right) \varepsilon \operatorname{diag}(t) v v^{\top}  = 0,
\end{equation}
where asterisk denotes entrywise products or powers, and $\mathbf{1}$ is the all-ones vector.  
Since $v \neq 0$, \eqref{eq:simplied-rk1} is equivalent to
 \begin{equation} \label{eq:better-simplified}
 \left( 3 \mu_2^2 {\|v\|^2}  -  \varepsilon \mu_2^2 \langle t, \mathbf{1} \rangle\right) v
    + \left( \mu_4 - 3 \mu_2^2 \right) v^{\ast 3} 
    + \left(-  \mu_4 +  \mu_2^2 \right) \varepsilon t \ast v  = 0.
\end{equation}
Automatically \eqref{eq:simplied-rk1}  holds at coordinates $i \in [n]$ where $v_i = 0$.
Thus we restrict \eqref{eq:simplied-rk1} to the support of $v$, that is, $\mathcal{I} = \{ i : v_i \neq 0\}$.  
Assuming $\mu_2 \neq 0$, \eqref{eq:simplied-rk1} implies  $v_{\mathcal{I}}, v^{\ast 3}_{\mathcal{I}}, t_{\mathcal{I}} \ast v_{\mathcal{I}}$ are linearly dependent, where subscripts indicate subvectors. 
Therefore for all $i, j, k \in \mathcal{I}$, 
\begin{equation} \label{}
    0 = \det \begin{pmatrix} v_i & v_i^3 & t_i v_i \\ v_j & v_j^3 & t_j v_j \\ v_k & v_k^3 & t_k v_k \end{pmatrix} = v_i v_j v_k \left( (t_j - t_k) v_i^2 + (-t_i + t_k) v_j^2 + (t_i - t_j) v_k^2 \right),
\end{equation}
or equivalently, 
\begin{equation} \label{eq:t-linear-system}
0 =  (t_j - t_k) v_i^2 + (-t_i + t_k) v_j^2 + (t_i - t_j) v_k^2.
\end{equation}

At this point, we constrain  $\mathcal{T}$ by assuming $t$ has distinct entries.
Now regard \eqref{eq:t-linear-system} as a $|\mathcal{I}|^3 \times |\mathcal{I}|$ linear system in $v_{\mathcal{I}}^{\ast 2}$. 
We claim the null space is spanned by $\mathbf{1}$ and $t_{\mathcal{I}}$.
Indeed, these lie in the null space.  
If $\mathcal{I} \leq 2$, they clearly span $\mathbb{R}^{\mathcal{I}}$ and thus the null space.  
If $\mathcal{I} \geq 3$, supposing $w \in \mathbb{R}^{\mathcal{I}}$ lies in the null space, fix any distinct $i, j \in \mathcal{I}$ and choose $b, c \in \mathbb{R}$ so that $w' := w - b \mathbf{1} - c t_{\mathcal{I}}$ has zero coordinates in positions $i$ and $j$.  From \eqref{eq:t-linear-system} it follows $(t_i - t_j) w'_k = 0$ for all $k \in \mathcal{I}$, whence $w' = 0$.  Again, $\mathbf{1}$ and $t_{\mathcal{I}}$ span the null space.

By the preceding paragraph, \eqref{eq:t-linear-system} implies there exist $\beta, \gamma \in \mathbb{R}$ such that
\begin{equation}\label{eq:v2-nice}
v_{\mathcal{I}}^{\ast 2} = \beta \mathbf{1} + \gamma t_{\mathcal{I}}.
\end{equation}
Plugging \eqref{eq:v2-nice} into \eqref{eq:simplied-rk1} by writing $v_{\mathcal{I}}^{\ast 3} = (\beta \mathbf{1} + \gamma t_{\mathcal{I}}) \ast v$, we obtain
\begin{equation} \label{eq:getting-good}
\left( 3 \mu_2^2 \| v \|^2 - \varepsilon \mu_2^2 \langle t, \mathbf{1} \rangle + \beta ( \mu_4 - 3 \mu_2^2) \right) v_{\mathcal{I}} + \left(  \varepsilon(-  \mu_4 + \mu_2^2) +  \gamma( \mu_4 - 3 \mu_2^2) \right)(t_{\mathcal{I}} \ast v_{\mathcal{I}}) = 0.
\end{equation}

Now break into cases.  Firstly suppose $|\mathcal{I}| \geq 2$.  By assumptions, the coefficients of $v_{\mathcal{I}}$ and $t_{\mathcal{I}} \ast v_{\mathcal{I}}$ in \eqref{eq:getting-good} must vanish:
\begin{align}
3 \mu_2^2 \| v \|^2 - \varepsilon \mu_2^2 \langle t, \mathbf{1} \rangle + \beta ( \mu_4 - 3 \mu_2^2) &= 0, \label{eq:big1} \\
\varepsilon(-  \mu_4 + \mu_2^2) +  \gamma( \mu_4 - 3 \mu_2^2) &= 0.  \label{eq:big2}
\end{align}
Taking inner product with $\mathbf{1}$, \eqref{eq:v2-nice} implies
\begin{equation} \label{eq:big3}
\|v\|^2 = \beta | \mathcal{I} | + \gamma \langle t_{\mathcal{I}}, \mathbf{1} \rangle,
\end{equation}
where we used $\|v\|^2 = \| v_{\mathcal{I}} \|^2$.
We regard \eqref{eq:big1}, \eqref{eq:big2}, \eqref{eq:big3} as a $3 \times 3$ linear system in $\|v\|^2, \beta, \gamma$.   Using \eqref{eq:big2} to obtain $\gamma$ and then \eqref{eq:big1}, \eqref{eq:big3} to solve for $\beta, \|v\|^2$, in particular  
\begin{equation} \label{eq:my-constants}
\gamma = \frac{\varepsilon (\mu_4 - \mu_2^2)}{\mu_4 - 3 \mu_2^2} \quad  \quad \, \text{and} \quad \quad \, 
\beta = \frac{\varepsilon \mu_2^2 \langle t, \mathbf{1} \rangle - 3 \mu_2^2 \gamma \langle t_{\mathcal{I}}, \mathbf{1}\rangle} {\mu_4 + 3(|\mathcal{I}| -1) \mu_2^2}
\end{equation}
assuming that the denominators do not vanish.
From \eqref{eq:v2-nice}, it is necessary that 
\begin{equation} \label{eq:inequality}
\beta + \gamma t_i \geq 0 \quad \quad \text{for all } i \in \mathcal{I}.
\end{equation}
If $\mu_4 \gg \mu_2^2$, then $\gamma \approx \varepsilon$ and $\beta \approx 0$ so $\beta + \gamma t_i \approx \varepsilon t_i$.  
Hence \eqref{eq:inequality}  holds if the coordinates of $t$ all have sign $\varepsilon$.  
More precisely, we constrain the distribution $\mathcal{D}$ to be so heavy-tailed that 
\begin{equation}\label{eq:key-assume}
\mu_4 \geq 10 n \mu_2^2 > 0
\end{equation}
holds.  We constrain teachers by setting
\begin{equation} \label{eq:another-key}
\mathcal{T} = \{ t' \in \mathbb{R}^n : \textup{ all coordinates of } t' \textup{ are distinct and positive and } 1 \leq \frac{\max_{i \in [n]} t'_i}{\min_{i \in [n]} t'_i}  \leq 2 \},
\end{equation}
and assuming $t \in \mathcal{T}$. 
Then, $1 \leq |\gamma| \leq \frac{9}{7}$ and by Cauchy-Schwarz, \eqref{eq:key-assume}, \eqref{eq:another-key},
\begin{equation}
|\beta| \leq \frac{\mu_2^2 \|t\| \sqrt{n} + 3 \mu_2^2(\frac{9}{7}) \|t \| \sqrt{n} }{\mu_4} \leq \frac{\mu_2^2 \|t\| \sqrt{n} (\frac{34}{7})}{10 n \mu_2^2} < \frac{1}{2} \frac{\|t\|}{ \sqrt{n}} \leq \min_{i \in [n]} t_i.
\end{equation}
It follows $\beta + \gamma t_i$ has sign $\varepsilon$ for all $i \in \mathcal{I}$.  
By \eqref{eq:inequality} and \eqref{eq:v2-nice},  $\varepsilon = 1$ and $v_i^2 = \beta + \gamma t_i$ for $i \in \mathcal{I}$.  
Thus, there exist $\tau_i \in \{1, -1, 0\}$ for $i \in [n]$ with $\tau_i \neq 0$ if and only if $i \in \mathcal{I}$ such that   
$v_i = \tau_i \sqrt{\beta + \gamma t_i} \in \mathbb{R} \setminus \{0\}$ where  $\beta, \gamma$ are as in \eqref{eq:my-constants}, and $S = \varepsilon v v^{\top} \in \mathbb{R}^{n \times n}$ has entries $\tau_i \tau_j \sqrt{(\beta + \gamma t_i)(\beta + \gamma t_j)}$.
Conversely, such $S$ indeed satisfies the criticality condition \eqref{eq:iid-crit}.
So depending on the choice of $\tau$, there is a critical point with support specified by $\mathcal{I}$.  
Choices $\tau$ and $-\tau$ give the same $S$ and otherwise the matrices $S$ are distinct.  
It follows there are $2^{|\mathcal{I}|-1}$ many critical points with support determined by $\mathcal{I}$.

The other case is  $|\mathcal{I}|=1$.  
Then, $v$ is a scalar multiple of $e_i$, where $\mathcal{I} = \{i\}$. 
Returning back to \eqref{eq:better-simplified}, it is easy to directly compute 
$S = \frac{\mu_2^2 \langle t, 1 \rangle - \mu_2^2 t_i + \mu_4 t_i}{\mu_4} e_i e_i^{\top}$.
This is real, nonzero and satisfies \eqref{eq:iid-crit}.  So again, there are $2^{|\mathcal{I}|-1}$ many critical points in this case.

Varying the support set in the preceding two paragraphs, we conclude that under assumptions \eqref{eq:key-assume} and \eqref{eq:another-key} the problem \eqref{eq:optim-problem} has the following number of critical points:
\begin{equation*} \label{eq:big-sum}
\frac{1}{2} \, \sum_{\substack{\mathcal{I} \subseteq \{1, \ldots, n\} \\ \mathcal{I} \neq \emptyset}} 2^{|\mathcal{I}|} \,\,\, = \,\,\, \frac{1}{2} \sum_{i=1}^n \binom{n}{i} 2^i \,\,\, = \,\,\, \frac{1}{2} (3^n - 1).
\end{equation*}
Here, the second equality is by the binomial theorem.
Finally, note that \eqref{eq:another-key} has positive {Lebesgue}-measure in $\mathbb{R}^n$.  The proof of \Cref{thm:iid-case} is complete.
\end{proof}

\section{Full polynomial from Example~\ref{ex:2x2-discriminant}}

\begin{align*}
&729 \mu_2^{24} t_{00}^{3} t_{11}^{3} + 1458 \mu_2^{22} \mu_4 t_{00}^{4} t_{11}^{2} + 1458 \mu_2^{22} \mu_4 t_{00}^{2} t_{11}^{4} - 2187 \mu_2^{20} \mu_4^{2} t_{00}^{4} t_{01}^{2} \\
&+ 972 \mu_2^{20} \mu_4^{2} t_{00}^{5} t_{11} + 486 \mu_2^{20} \mu_4^{2} t_{00}^{2} t_{01}^{2} t_{11}^{2} + 1458 \mu_2^{20} \mu_4^{2} t_{00}^{3} t_{11}^{3} - 2187 \mu_2^{20} \mu_4^{2} t_{01}^{2} t_{11}^{4} \\
&+ 972 \mu_2^{20} \mu_4^{2} t_{00} t_{11}^{5} + 216 \mu_2^{18} \mu_4^{3} t_{00}^{6} - 5184 \mu_2^{18} \mu_4^{3} t_{00}^{3} t_{01}^{2} t_{11} - 486 \mu_2^{18} \mu_4^{3} t_{00}^{4} t_{11}^{2} \\
&- 5184 \mu_2^{18} \mu_4^{3} t_{00} t_{01}^{2} t_{11}^{3} - 486 \mu_2^{18} \mu_4^{3} t_{00}^{2} t_{11}^{4} + 216 \mu_2^{18} \mu_4^{3} t_{11}^{6} + 3132 \mu_2^{16} \mu_4^{4} t_{00}^{4} t_{01}^{2} \\
&- 864 \mu_2^{16} \mu_4^{4} t_{00}^{5} t_{11} + 6912 \mu_2^{16} \mu_4^{4} t_{00} t_{01}^{4} t_{11} - 11448 \mu_2^{16} \mu_4^{4} t_{00}^{2} t_{01}^{2} t_{11}^{2} - 1377 \mu_2^{16} \mu_4^{4} t_{00}^{3} t_{11}^{3} \\
&+ 3132 \mu_2^{16} \mu_4^{4} t_{01}^{2} t_{11}^{4} - 864 \mu_2^{16} \mu_4^{4} t_{00} t_{11}^{5} - 216 \mu_2^{14} \mu_4^{5} t_{00}^{6} + 4608 \mu_2^{14} \mu_4^{5} t_{00}^{2} t_{01}^{4} \\
&+ 2880 \mu_2^{14} \mu_4^{5} t_{00}^{3} t_{01}^{2} t_{11} - 36 \mu_2^{14} \mu_4^{5} t_{00}^{4} t_{11}^{2} + 4608 \mu_2^{14} \mu_4^{5} t_{01}^{4} t_{11}^{2} + 2880 \mu_2^{14} \mu_4^{5} t_{00} t_{01}^{2} t_{11}^{3} \\
&- 36 \mu_2^{14} \mu_4^{5} t_{00}^{2} t_{11}^{4} - 216 \mu_2^{14} \mu_4^{5} t_{11}^{6} - 2034 \mu_2^{12} \mu_4^{6} t_{00}^{4} t_{01}^{2} - 4096 \mu_2^{12} \mu_4^{6} t_{01}^{6} \\
&+ 360 \mu_2^{12} \mu_4^{6} t_{00}^{5} t_{11} - 1536 \mu_2^{12} \mu_4^{6} t_{00} t_{01}^{4} t_{11} + 7620 \mu_2^{12} \mu_4^{6} t_{00}^{2} t_{01}^{2} t_{11}^{2} + 604 \mu_2^{12} \mu_4^{6} t_{00}^{3} t_{11}^{3} \\
&- 2034 \mu_2^{12} \mu_4^{6} t_{01}^{2} t_{11}^{4} + 360 \mu_2^{12} \mu_4^{6} t_{00} t_{11}^{5} + 72 \mu_2^{10} \mu_4^{7} t_{00}^{6} - 1536 \mu_2^{10} \mu_4^{7} t_{00}^{2} t_{01}^{4} \\
&- 960 \mu_2^{10} \mu_4^{7} t_{00}^{3} t_{01}^{2} t_{11} + 12 \mu_2^{10} \mu_4^{7} t_{00}^{4} t_{11}^{2} - 1536 \mu_2^{10} \mu_4^{7} t_{01}^{4} t_{11}^{2} - 960 \mu_2^{10} \mu_4^{7} t_{00} t_{01}^{2} t_{11}^{3} \\
&+ 12 \mu_2^{10} \mu_4^{7} t_{00}^{2} t_{11}^{4} + 72 \mu_2^{10} \mu_4^{7} t_{11}^{6} + 348 \mu_2^{8} \mu_4^{8} t_{00}^{4} t_{01}^{2} - 96 \mu_2^{8} \mu_4^{8} t_{00}^{5} t_{11} + 768 \mu_2^{8} \mu_4^{8} t_{00} t_{01}^{4} t_{11} \\
&- 1272 \mu_2^{8} \mu_4^{8} t_{00}^{2} t_{01}^{2} t_{11}^{2} - 153 \mu_2^{8} \mu_4^{8} t_{00}^{3} t_{11}^{3} + 348 \mu_2^{8} \mu_4^{8} t_{01}^{2} t_{11}^{4} - 96 \mu_2^{8} \mu_4^{8} t_{00} t_{11}^{5} \\
&- 8 \mu_2^{6} \mu_4^{9} t_{00}^{6} + 192 \mu_2^{6} \mu_4^{9} t_{00}^{3} t_{01}^{2} t_{11} + 18 \mu_2^{6} \mu_4^{9} t_{00}^{4} t_{11}^{2} + 192 \mu_2^{6} \mu_4^{9} t_{00} t_{01}^{2} t_{11}^{3} \\
&+ 18 \mu_2^{6} \mu_4^{9} t_{00}^{2} t_{11}^{4} - 8 \mu_2^{6} \mu_4^{9} t_{11}^{6} - 27 \mu_2^{4} \mu_4^{10} t_{00}^{4} t_{01}^{2} + 12 \mu_2^{4} \mu_4^{10} t_{00}^{5} t_{11} + 6 \mu_2^{4} \mu_4^{10} t_{00}^{2} t_{01}^{2} t_{11}^{2} \\
&+ 18 \mu_2^{4} \mu_4^{10} t_{00}^{3} t_{11}^{3} - 27 \mu_2^{4} \mu_4^{10} t_{01}^{2} t_{11}^{4} + 12 \mu_2^{4} \mu_4^{10} t_{00} t_{11}^{5} - 6 \mu_2^{2} \mu_4^{11} t_{00}^{4} t_{11}^{2} - 6 \mu_2^{2} \mu_4^{11} t_{00}^{2} t_{11}^{4} \\
&+ \mu_4^{12} t_{00}^{3} t_{11}^{3}
\end{align*}

\end{document}